%% file: icml_main.tex
\documentclass[nohyperref]{article}
\usepackage[numbers]{natbib}
\usepackage[accepted]{icml2022}
\usepackage{etoolbox}
\newtoggle{icml}
\toggletrue{icml}
\usepackage{xcolor}  
\usepackage{color-edits}  

\usepackage[textsize=tiny]{todonotes}
\input{preamble/header}

\icmltitlerunning{Do Differentiable Simulators Give Better Policy Gradients?}

\begin{document}

\twocolumn[
\icmltitle{
Do Differentiable Simulators Give Better Policy Gradients?

}



\icmlsetsymbol{equal}{*}

\begin{icmlauthorlist}
\icmlauthor{H.J. Terry Suh}{yyy}
\icmlauthor{Max Simchowitz}{yyy}
\icmlauthor{Kaiqing Zhang}{yyy}
\icmlauthor{Russ Tedrake}{yyy}
\end{icmlauthorlist}

\icmlaffiliation{yyy}{Department of Electrical Engineering and Computer Science, Massachusetts Institute of Technology, Cambridge, USA}

\icmlcorrespondingauthor{H.J.Terry Suh}{hjsuh@mit.edu}

\icmlkeywords{Differentiable Simulation, Policy Search, Policy Gradients}

\vskip 0.3in
]



\printAffiliationsAndNotice{}  

\begin{abstract}
\input{body/abstract}

\end{abstract}
\input{body/introduction}
\input{body/prelim_two}

\input{body/bias_variance}
\input{body/interpolation}
\input{body/case_studies}
\input{body/discussion}

\input{body/conclusion}

\bibliography{references}
\bibliographystyle{icml2022}

\appendix
\onecolumn

\input{appendix/max_derivations}

\input{appendix/max_ders_two}

\input{appendix/derivations_interpolation}
\end{document}

%% file: preamble/header.tex
\usepackage{microtype}
\usepackage{graphicx}
\usepackage{subfigure}
\usepackage{booktabs} 
\usepackage{upgreek}
\iftoggle{icml}
{}
{
	\usepackage[breaklinks=true]{hyperref}
  \usepackage[ruled,lined]{algorithm2e}
\usepackage{fullpage}
}

\usepackage{hyperref}

\addauthor{ts}{red}
\addauthor{ms}{purple}
\addauthor{kz}{cyan}
\addauthor{tp}{orange}

\usepackage{amsmath}
\usepackage{amsfonts}
\usepackage{amsthm}
\usepackage{thmtools}
\usepackage{amssymb}
\usepackage{ulem}
\usepackage{mathtools}
\usepackage{dsfont}
\usepackage{verbatim}
\usepackage{graphicx}
\usepackage{soul}
\usepackage{mdframed}
\usepackage{bbm}
\usepackage{sectsty}
\usepackage{xcolor}
\usepackage{appendix}
\usepackage{mathrsfs}
\usepackage[capitalize,noabbrev]{cleveref}
\newcommand{\Exp}{\mathbb{E}}
\newcommand{\R}{\mathbb{R}}
\newcommand{\N}{\mathbb{N}}
\newcommand{\calN}{\mathcal{N}}
\newcommand{\bz}{\mathbf{z}}
\newcommand{\nablatwo}{\nabla^{2}\,}
\newcommand{\nabthet}{\nabla_{\btheta}}
\newcommand{\fro}{\mathrm{F}}
\newcommand{\rmd}{\mathrm{d}}
\newcommand{\op}{\mathrm{op}}
\newcommand{\rmD}{\mathrm{D}}
\newcommand{\Expw}{\Exp_{\bw \sim p}}

\newcommand{\Variance}{\normalfont{\textbf{Var}}}
\newcommand{\sighatsq}{\hat{\sigma}^2}

\newcommand{\iidsim}{\overset{\mathrm{i.i.d.}}{\sim}}
\newcommand{\gin}{g_{\mathrm{in}}}
\newcommand{\gout}{g_{\mathrm{out}}}
\newcommand{\namefont}[1]{{\normalfont #1}}
\newcommand{\fobg}{\namefont{FoBG}}

\newcommand{\zobg}{\namefont{ZoBG}}
\newcommand{\barbw}{\bar{\bw}}

\newcommand{\bx}{\mathbf{x}}
\newcommand{\bw}{\mathbf{w}}
\newcommand{\bu}{\mathbf{u}}
\newcommand{\btheta}{\boldsymbol{\theta}}
\newcommand{\dirac}{\updelta}

\newcommand{\nabhat}{\hat{\nabla}}
\newcommand{\nabbar}{\bar{\nabla}}
\newcommand{\nabhatzero}{\nabhat^{[0]}}
\newcommand{\nabhatone}{\nabhat^{[1]}}
\newcommand{\nabbarone}{\nabbar^{[1]}}
\newcommand{\nabbarzero}{\nabbar^{[0]}}

\theoremstyle{plain}
\newtheorem{theorem}{Theorem}[section]
\newtheorem{proposition}[theorem]{Proposition}
\newtheorem{lemma}[theorem]{Lemma}
\newtheorem{claim}[theorem]{Claim}

\newtheorem{corollary}[theorem]{Corollary}
\theoremstyle{definition}
\newtheorem{definition}[theorem]{Definition}
\newtheorem{assumption}[theorem]{Assumption}
\theoremstyle{remark}
\newtheorem{remark}[theorem]{Remark}
\newtheorem{example}[theorem]{Example}

\crefname{equation}{eq}{eqs}
\Crefname{equation}{Eq}{Eqs}
\Crefname{claim}{Claim}{Claims}

\crefname{lemma}{lemma}{lemmas}
\Crefname{lemma}{Lemma}{Lemmas}
\crefname{assumption}{assumption}{assumptions}
\Crefname{assumption}{Assumption}{Assumption}

%% file: body/abstract.tex


Differentiable simulators promise faster computation time for reinforcement learning by replacing zeroth-order gradient estimates of a stochastic objective with an estimate based on first-order gradients. However, it is yet unclear what factors decide the performance of the two estimators on complex landscapes that involve long-horizon planning and control on physical systems, despite the crucial relevance of this question for the utility of differentiable simulators. We show that characteristics of certain physical systems, such as stiffness or discontinuities, may compromise the efficacy of the first-order estimator, and analyze this phenomenon through the lens of bias and variance. We additionally propose an $\alpha$-order gradient estimator, with $\alpha \in [0,1]$, which correctly utilizes exact gradients to combine the efficiency of first-order estimates with the robustness of zero-order methods. We demonstrate the pitfalls of traditional estimators and the advantages of the $\alpha$-order estimator on some numerical examples.

%% file: body/introduction.tex
\section{Introduction}\label{sec:introduction}
\vskip -0.15 true in
Consider the problem of minimizing a \emph{stochastic objective},
\begin{align*}
    \min_{\btheta} F(\btheta) = \min_{\btheta} \Exp_{\bw} f(\btheta,\bw).
\end{align*}
At the heart of many algorithms for reinforcement learning (RL) lies  \emph{zeroth-order} estimation of the gradient $\nabla F$ \cite{policygradienttheorem,ppo}. Yet, in domains that deal with structured systems, such as linear control, physical simulation, or robotics, it is possible to obtain \emph{exact} gradients of $f$, which can also be used to construct a \emph{first-order} estimate of $\nabla F$. The availability of both options begs the question: given access to exact gradients of $f$, which estimator should we prefer?

In stochastic optimization, the theoretical benefits of using first-order estimates of $\nabla F$ over zeroth-order ones have mainly been understood through the lens of variance and convergence rates \cite{ghadimilan,montecarlo}: the first-order estimator often (\emph{not always}) results in much less variance compared to the zeroth-order one, which leads to faster convergence rates to a local minima of general nonconvex smooth objective functions.

However, the landscape of RL objectives that involve long-horizon sequential decision making (e.g. policy optimization) is challenging to analyze, and convergence properties in these landscapes are relatively poorly understood, except for structured settings such as finite-state MDPs \cite{agarwal,kaiqing} or linear control \cite{fazel,russo}. In particular, physical systems with contact, as we show in \Cref{fig:systems}, can display complex characteristics including nonlinearities, non-smoothness, and discontinuities \cite{hybridsystems, robotmanipulation, suh2021bundled}.

\begin{figure}[t]
\centering\includegraphics[width = 0.48\textwidth]{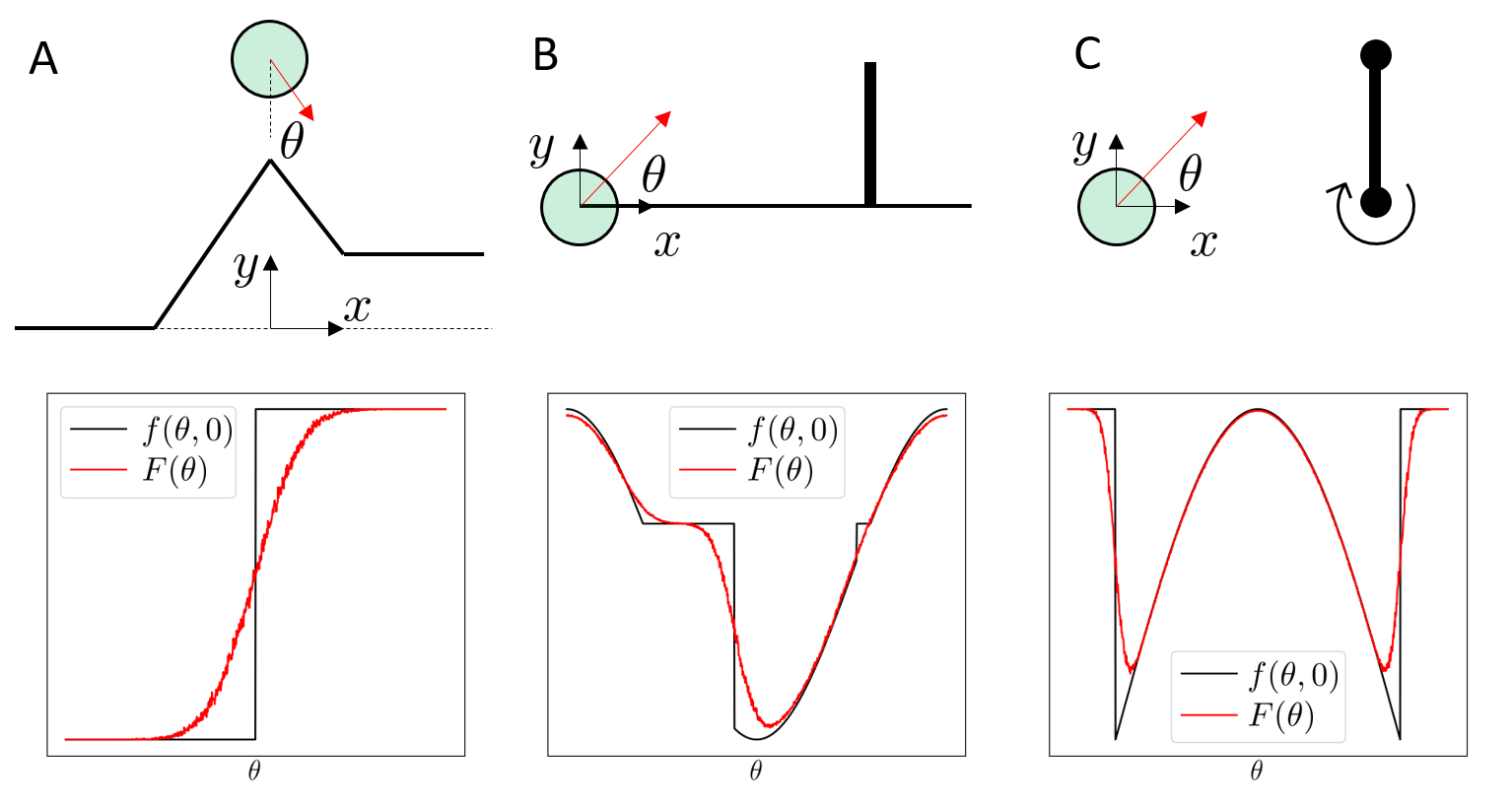}
\caption{Examples of simple optimization problems on physical systems. Goal is to: A. maximize $y$ position of the ball after dropping. B. maximize distance thrown, with a wall that results in inelastic impact. C. maximize transferred angular momentum to the pivoting bar through collision. Second row: the original objective and the stochastic objective after randomized smoothing.}
\label{fig:systems}
\vskip -0.28 true in
\end{figure}

Nevertheless, lessons from convergence rate analysis tell us that there may be benefits to using the exact gradients even for these complex physical systems. Such ideas have been championed through the term ``differentiable simulation'', where forward simulation of physics is programmed in a manner that is consistent with automatic differentiation \cite{brax,difftaichi,drake,werling,add,dojo}, or computation of analytic derivatives \cite{pinocchio}. These methods have shown promising results in decreasing computation time compared to zeroth-order methods \cite{plasticinelab,brax,deluca,d3pg,filipe,pods}.


Existing literature in differentiable simulation mainly focuses on the use of exact gradients for \emph{deterministic} optimization. However, \cite{suh2021bundled, randomizedsmoothing} show that using exact gradients for a deterministic objective can lead to suboptimal behavior of certain systems due to their landscapes. In these systems, stochasticity can be used to \emph{regularize} the landscapes with randomized smoothing \cite{duchi1}. We illustrate how the landscapes change upon injecting noise (\Cref{fig:systems}), and list some benefits of considering a \emph{surrogate} stochastic objective.

\begin{itemize}\itemsep -0.2em
\item \textbf{Stochasticity smooths local minima}. As noted in \cite{suh2021bundled,metz2021gradients}, stochasticity can alleviate some of the high-frequency local minima that deterministic gradients will be stuck on. For instance, the small discontinuity on the right side of Figure \ref{fig:systems}.B is filtered by Gaussian smoothing.
\item \textbf{Stochasticity alleviates flat regions}. In systems of Figure \ref{fig:systems}, the gradients in some of the regions can be completely flat. This stalls progress of gradient descent. The stochastic objective, however, still has non-zero gradient as some samples escape the flat regions and provide an informative direction of improvement.
\item \textbf{Stochasticity encodes robustness.} In Figure \ref{fig:systems}.C, following the gradient to increase the transferred momentum causes the ball to miss the pivot and land in a high-cost region. In contrast, the stochastic objective has a local minimum within the safe region, as the samples provide information about missing the pivot.
\end{itemize}

Thus, our work attempts to compare two versions of gradient estimators in the stochastic setting: the first-order estimator and the zeroth-order one. This setting rules out the case that zeroth-order estimates perform better simply because of stochasticity, and sets equal footing for the two methods.

When $f$ is continuous, these quantities both converge to the same quantity ($\nabla F$) in expectation. We first show that even with continuous $f$, the first-order gradient estimate \emph{can} result in more variance than the zeroth-order one due to the \emph{stiffness} of dynamics or due to compounding of gradients in chaotic systems \cite{parmas,metz2021gradients}.

In addition, we show that the assumption of continuous $f$ can be violated in many relevant physical systems that are nearly/strictly \emph{discontinuous} in the underlying landscape. These discontinuities are commonly caused by contact and geometrical constraints. We provide minimal examples to highlight specific challenges in Figure \ref{fig:systems}. These are not mere pathologies, but abstractions of more complicated examples that are rich with contact, such as robotic manipulation.

We show that the presence of such discontinuities causes the first-order gradient estimator to be {\it biased}, while the zeroth-order one still remains unbiased under discontinuities. Furthermore, we show that stiff continuous approximations of discontinuities, even if asymptotically unbiased, can still suffer from what we call \emph{empirical bias} under finite-sample settings. This results in a bias-variance tradeoff between the biased first-order estimator and the \emph{often} high-variance, yet unbiased zeroth-order estimator. Intriguingly, we find that the bias-variance tradeoff in this setting manifests itself not through convergence rates, but through different local minima. This shows that the two estimators may fundamentally operate on different landscapes implicitly.

The presence of discontinuities need not indicate that we need to commit ourselves to uniformly using one of the estimators. Many physical systems are \emph{hybrid} by nature \cite{hybridsystems}; they consist of smooth regions that are separated by manifolds of non-smoothness or discontinuities. This suggests that we may be able to utilize the first-order estimates far away from these manifolds to obtain benefits of convergence rates, while switching to zeroth-order ones in the vicinity of discontinuities to obtain unbiased estimates.

For this purpose, we further attempt to answer the question: how can we then correctly utilize exact gradients of $f$ for variance reduction when we know the objective is nearly discontinuous? Previous works show that the two estimators can be combined by interpolating based on empirical variance \cite{parmas,pathologies,montecarlo}. However, we show that in the presence of near-discontinuities, selecting based on empirical variance alone can lead to highly inaccurate estimates of $\nabla F$, and propose a robustness constraint on the accuracy of the interpolated estimate to remedy this effect.

\textbf{Contributions.} We 1) shed light on some of the inherent problems of RL using differentiable simulators, and answer which gradient estimator can be more useful under different characteristics of underlying systems such as discontinuities, stiffness, and chaos; and 2) present the  $\alpha$-order  gradient estimator, a robust interpolation strategy between the two gradient estimators that utilizes exact gradients without falling into the identified pitfalls of the previous methods. 

We hope both contributions inspire algorithms for policy optimization using differentiable simulators, as well as design guidelines for new and existing simulators. 

%% file: body/prelim_two.tex

\section{Preliminaries} 
\vskip -0.15 true in
\textbf{Notation. } We denote the expectation of a random vector $\bz$ as $\Exp[\bz]$, and its variance  as $\Variance[\bz] := \Exp[\|\bz - \Exp[\bz]\|^2]$. Expectations are defined in almost-sure sense, so that the law of large numbers holds (see \Cref{app:formal_prelim} for details).

\textbf{Setting. } We study a discrete-time, finite-horizon, continuous-state control problem with states $\bx \in \R^n$, inputs $\bu \in \R^m$,  transition function $\phi:\R^n \times \R^m \to \R^n$, and horizon $H \in \N$. Given a sequence of costs $c_h:\R^n \times \R^m \to \R$, a family of policies $\pi_h:\R^n \times \R^d \to \R^m$ parameterized by $\btheta \in \R^d$, and a sequence of injected noise terms $\bw_{1:H} \in (\R^m)^H$, we define the cost-to-go functions
 \begin{align*}
&\qquad V_h(\bx_h,\bw_{h:H},\btheta) = \textstyle \sum_{h' = h}^H c_h(\bx_{h'},\bu_{h'}), \quad \\
\text{s.t. } &\bx_{h'+1} = \phi(\bx_{h'},\bu_{h'}), \,\, \bu_{h
'}= \pi(\bx_{h'},\btheta) +  \bw_{h'}, \,\, h' \ge h.
\end{align*}
Our aim is to minimize the policy optimization objective 
\begin{align}
F(\btheta) := \Exp_{\bx_1 \sim \rho} \Exp_{\bw_{h}\iidsim p} V_1(\bx_1,\bw_{1:H},\btheta), \label{eq:F_theta_gen}
\end{align}
where $\rho$ is a distribution over initial states $\bx_1$, and  $\bw_{1},\dots,\bw_{H}$ are independent and identically distributed according to some distribution $p$. In the main text, we make the following assumption on the distributions $\rho$ and $p$:

\begin{assumption}\label{asm:dist}
We assume that $\rho$ has finite moments, and that $p = \mathcal{N}(0,\sigma^2 I_n)$ for some $\sigma > 0$.
\end{assumption}
Our rationale for Gaussian $p$ is that we view $\bw_{1:H}$ as \emph{smoothing} to regularize the optimization landscape \cite{duchi2, gradientapproximation}. To simplify the main text, we take $\bx_1$ to be deterministic ($\rho$ is a dirac-delta), with general $\rho$ being addressed in the appendix. Setting $\bar{\bw} = \bw_{1:H}$, $\bar{p} = \calN(0,\sigma^2 I_{nH})$, and $f(\btheta,\barbw)= V_1(\bx_1,\bar{\bw},\btheta)$, we can express $F(\btheta)$ as a \emph{stochastic optimization problem},
\begin{align*}
F(\btheta) := \Exp_{\barbw\sim \bar{p}} f(\btheta,\barbw).
\end{align*}

\textbf{Trajectory optimization.} Our parametrization also includes open-loop trajectory optimization. Letting the policy parameters be an open-loop sequence of inputs $\btheta = \{\btheta_h\}_{h=1}^H$ and having no feedback $\pi(\bx_h,\btheta)=\btheta_h$, we optimize over sequence of inputs to be applied to the system.

\textbf{One-step optimization.} We illustrate some key ideas in the open-loop case where $H=1$: $\pi:\mathbb{R}^d\rightarrow\mathbb{R}^m$ is the identity function with $\bar{\bw} = \bw \in \R^m$, $d=m$ and $c: \R^m \to \R$, 
\begin{align}
 F(\btheta) = \Exp_{\bw \sim p}f(\btheta,\bw), \quad   f(\btheta,\bw) = c(\btheta +  \bw).
\end{align}

\subsection{Gradient Estimators}
\vskip -0.1 true in
In order to minimize $F(\btheta)$, we consider iterative optimization using stochastic estimators of its gradient $\nabla F(\btheta)$. We say a function $\psi: \R^{d_1} \to \R^{d_2}$ has \emph{polynomial growth} if there exist constants $a,b$ such that, for all $\bz \in \R^{d_1}$, $\|\psi(\bz)\| \le a(1+\|\bz\|^b)$. The following assumption ensures these gradients are well-defined.
\begin{assumption}\label{asm:polynomial}
We assume that the policy $\pi$ is continuously differentiable everywhere, and the dynamics $\phi$, as well as the cost $c_h$ have polynomial growth.
\end{assumption}

Even when the costs or dynamics are \emph{not} differentiable, the expected cost $F(\btheta)$ is differentiable due to the smoothing $\barbw$. $\nabla F(\btheta)$ is referred to as the \emph{policy gradient}. 

\textbf{Zeroth-order estimator.} The policy gradient can be estimated only using samples of the function values.
\begin{definition}  
    Given a single zeroth-order estimate of the policy gradient $\nabhatzero F_i(\btheta)$, we define the zeroth-order batched gradient (\zobg{}) $\nabbarzero F(\btheta)$ as the sample mean,
    \begin{align*}\small
        \nabhatzero F_i(\btheta) & \coloneqq \frac{1}{\sigma^2}V_1(\bx_1, \bw^i_{1:H},\btheta) \bigg[\sum^H_{h=1} \rmD_{\btheta} \pi(\bx_h^i,\btheta)^\intercal \bw^i_h\bigg] \\\small
        \nabbarzero F(\btheta) & \coloneqq \textstyle \frac{1}{N} \sum^N_{i=1} \nabhatzero F_i(\btheta),
    \end{align*}
    where $\bx^i_h$ is the state at time $h$ of a trajectory induced by the noise $\bw^i_{1:H}$, $i$ is the index of the sample trajectory, and $\rmD_{\btheta} \pi$ is the Jacobian matrix $\partial \pi/\partial \btheta\in\mathbb{R}^{m\times d}$.
\end{definition}

The hat notation denotes a per-sample Monte-Carlo estimate, and bar-notation a sample mean. The \zobg{} is also referred to as the REINFORCE \cite{reinforce}, score function, or the likelihood-ratio gradient. 

\textbf{Baseline.} In practice, a baseline term $b$ is subtracted from $V_1(\bx_1,\bw^i_{1:H},\btheta)$ for variance reduction. We use the zero-noise rollout as the baseline $b=V_1(\bx_1,\mathbf{0}_{1:H},\btheta)$ \cite{gradientapproximation}:
\begin{align*}\small
    \frac{1}{\sigma^2}\bigg[V_1(\bx_1, \bw^i_{1:H},\btheta) - b\bigg] \bigg[\sum^H_{h=1} \rmD_{\btheta} \pi(\bx_h^i,\btheta)^\intercal \bw^i_h\bigg].
\end{align*}
\textbf{First-order estimator.}  In differentiable simulators,  the gradients of the dynamics $\phi$ and costs $c_h$ are available \emph{almost surely} (i.e., with probability one). Hence, one may compute the exact gradient $\nabla_{\btheta} V_1(\bx_1,\bw_{1:H},\btheta)$ by automatic differentiation and average them to estimate $\nabla F(\btheta)$.
\begin{definition} Given a single first-order gradient estimate $\nabhatone F_i(\btheta)$, we define the first-order batched gradient (\fobg{}) as the sample mean:
    \begin{align*}\small
        \nabhatone F_i(\btheta) & \coloneqq \nabla_{\btheta} V_1(\bx_1, \bw_{1:H}^i, \btheta) \\\small
        \nabbarone F(\btheta) & \coloneqq \textstyle \frac{1}{N} \sum^N_{i=1} \nabhatone F_i(\btheta).
    \end{align*}
\end{definition}
The \fobg{} is also referred to as the reparametrization gradient \cite{kingma}, or the pathwise derivative \cite{stochasticcompuationgraph}. Finally, we define the empirical variance.

\begin{definition}[Empirical variance] For $k \in \{0,1\}$, we define the empirical variance by
    \begin{align*}
        \sighatsq_{k} &= \textstyle \frac{1}{N-1}\sum_{i=1}^N\|\nabhat^{[k]} F_i(\btheta)  - \nabbar^{[k]} F(\btheta)\|^2.
        \end{align*}
\end{definition}



%% file: body/bias_variance.tex

\newcommand{\cE}{\mathcal{E}}
\newcommand{\I}{\mathbb{I}}
\newcommand{\erf}{\mathrm{erf}}

\section{Pitfalls of First-order Estimates}\label{sec:bias_variance}

What are the cases for which we would prefer to use the \zobg{} over the \fobg{} in policy optimization using differentiable simulators? Throughout this section, we analyze the performance of the two estimators through their bias and variance properties, and find pathologies where using the first-order estimator blindly results in worse performance.

\subsection{Bias under discontinuities}

Under standard regularity conditions, it is well-known that both estimators are unbiased estimators of the true gradient $\nabla F(\btheta)$. However, care must be taken to define these conditions precisely. Fortunately,  the \zobg{} is still unbiased under mild assumptions.

\begin{lemma}\label{lem:ZOBG_bias} 
Under \Cref{asm:dist} and \Cref{asm:polynomial}, the \zobg{} is an unbiased estimator of the stochastic objective.
    \begin{align*}
        \Exp[\nabbarzero F(\btheta)] = \nabla F(\btheta).  
    \end{align*}
\end{lemma}

In contrast, the FoBG requires strong continuity conditions in order to satisfy the requirement for unbiasedness. However, under Lipschitz continuity, it is indeed unbiased.

\begin{lemma}\label{lem:FOBG_bias} 
Under \Cref{asm:dist} and \Cref{asm:polynomial}, and  if $\phi(\cdot,\cdot)$ is locally Lipschitz and $c_h(\cdot,\cdot)$ is continuously differentiable, then $\nabbarone F(\btheta)$ is defined almost surely, and
\begin{align*}
    \Exp[\nabbarone F(\btheta)] = \nabla F(\btheta).
\end{align*}
\end{lemma}

The proofs and more rigorous statements of both lemmas are provided in \Cref{app:formal}. Notice that \Cref{lem:ZOBG_bias} permits $V_h$ to have discontinuities (via discontinuities of $c_h$ and $\phi$), whereas \Cref{lem:FOBG_bias} does not.

\paragraph{Bias of \fobg{} under discontinuities.} The \fobg{} can fail when applied to discontinuous landscapes. We illustrate a simple case of biasedness through a counterexample. 
\begin{example}[\textbf{Heaviside}]\label{exmp:heavisied}\cite{teg,suh2021bundled} Consider the Heaviside function,
    \begin{align*}
        f(\btheta,\bw) = H(\btheta+\bw), \quad H(t) = \begin{cases}
         1 & t \geq 0 \\
         0 & t < 0 
        \end{cases},
    \end{align*}
    whose stochastic objective becomes the error function
    \begin{align*}
        F(\btheta) = \Exp_{\bw}[H(\btheta + \bw)] & = \erf(-\btheta; \sigma^2 ),
    \end{align*}
    where $\erf(t; \sigma^2 ) := \int_{t}^{\infty} \frac{1}{\sqrt{2\pi}\sigma}e^{-x^2/\sigma^2}\rmd x$ is the Gaussian tail integral. Defining the gradient of the Monte-Carlo objective $H(\btheta + \bw)$ requires subtlety. It is common in physics to define $\nabla_{\btheta} H(\btheta + \bw) = \dirac(\btheta+\bw)$ as a dirac-delta function, where integration is interpreted so that the fundamental theorem of calculus holds. This is \emph{irreconcilable} with using \emph{expectation} to define the integral, which presupposes that the law of large numbers hold. Indeed, since $\nabla_{\btheta} H(\btheta + \bw) = 0$ for all $\btheta \ne - \bw$,
    we have $\Exp_{\bw_i} \delta(\btheta + \bw_i) = 0$. Hence, the \fobg{} is biased, because the gradient of the stochastic objective at any $\btheta$ is non-zero: $\nabla_{\btheta} \text{erf}(-\btheta;\sigma^2) = \frac{1}{\sqrt{2\pi}\sigma}\exp(-(\btheta-\bw)/2\sigma^2)\ne 0$.
    
    It is worth noting that the empirical variance of the \fobg{} estimator in this example is zero, since all the samples are identically zero. On the other hand, the \zobg{} escapes this problem and provides an unbiased estimate, since it always takes finite intervals that include the integral of the delta. 
    \begin{figure}[thpb]
	\centering\includegraphics[width = 0.48\textwidth]{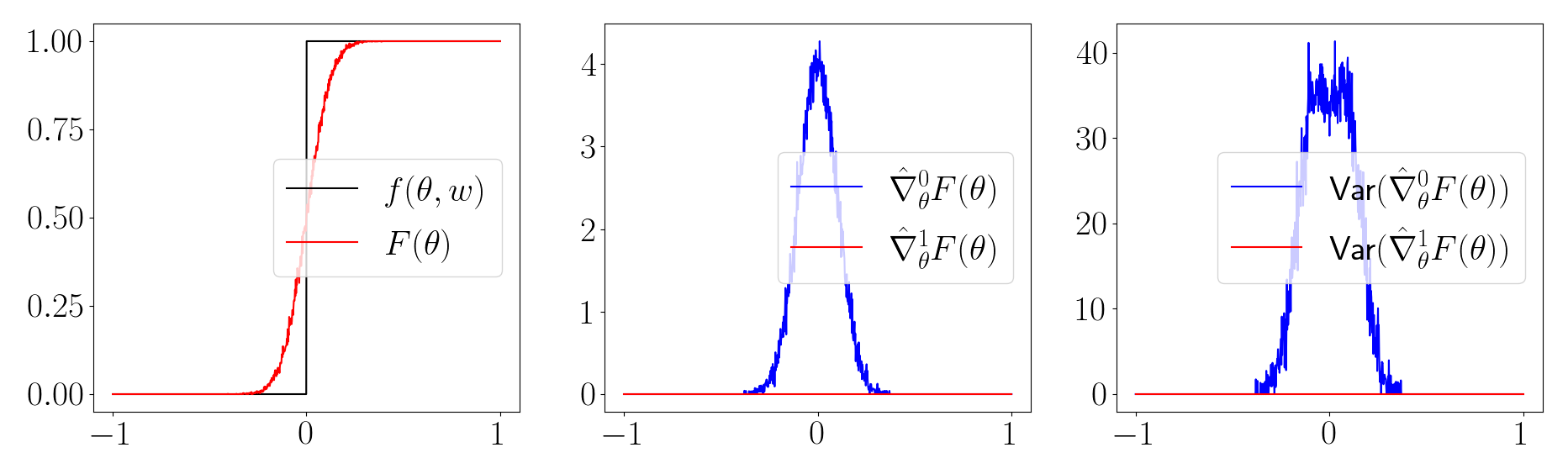}
    \caption{From left: heaviside objective $f(\btheta,\bw)$ and stochastic objective $F(\btheta)$, empirical values of the gradient estimates, and their empirical variance.}
	\label{fig:heaviside}
	\vskip -0.2 true in
    \end{figure}
\end{example}

\subsection{The ``Empirical bias'' phenomenon}\label{sec:empiricalbias}
One might argue that \emph{strict} discontinuity is simply an artifact of modeling choice in simulators; indeed, many simulators approximate discontinuous dynamics as a limit of continuous ones with growing Lipschitz constant \cite{add,Elandt2019APF}. In this section, we explain how this can lead to a phenomenon we call \emph{empirical bias}, where the \fobg{} appears to have low empirical variance, but is still highly inaccurate; i.e. it ``looks'' biased when a  finite number of samples are used. Through this phenomenon, we claim that performance degradation of  first-order gradient estimates do not require strict discontinuity, but is also present in continuous, yet \emph{stiff} approximations of discontinuities. \footnote{We say that a continuous function $f$ is stiff at $x$ if the magnitude of the gradient $\|\nabla f(x)\|$ is high.}

\begin{definition}[Empirical bias] Let $\bz$ be a vector-valued random variable with $\Exp[\|\bz\|] < \infty$. We say $\bz$ has $(\beta,\Delta,S)$-empirical bias if there is a random event $\cE$ such that $\Pr[\cE] \ge 1-\beta$,  and $\|\Exp[\bz \mid \cE] - \Exp[\bz]\| \ge \Delta$, but $\|\bz - \Exp[\bz \mid \cE]\| \le S$ almost surely on $\cE$.
\end{definition}
A paradigmatic example of empirical bias is a random scalar $\bz$ which takes the value $0$ with probability $1-\beta$, and $\frac{1}{\beta}$ with probability $\beta$. Setting $\cE = \{\bz = 0\}$, we see $\Exp[\bz] = 1$, $\Exp[\bz \mid \cE] = 0$, and so $\bz$ satisfies $(\beta,1,0)$-empirical bias. Note that $\Variance[\bz] = 1/\beta-1$; in fact, small-$\beta$ empirical bias implies large variance more generally.

\begin{lemma}\label{lem:variance_empirical_bias} Suppose $\bz$ has $(\beta,\Delta,S)$-empirical bias. Then $\Variance[\bz] \ge \frac{\Delta_0^2}{\beta}$, where $\Delta_0 := \max\{0,(1-\beta)\Delta - \beta \|\Exp[\bz]\|\}$. 
\end{lemma}

\begin{figure}[b]
\vskip -0.2 true in
\centering\includegraphics[width = 0.5\textwidth]{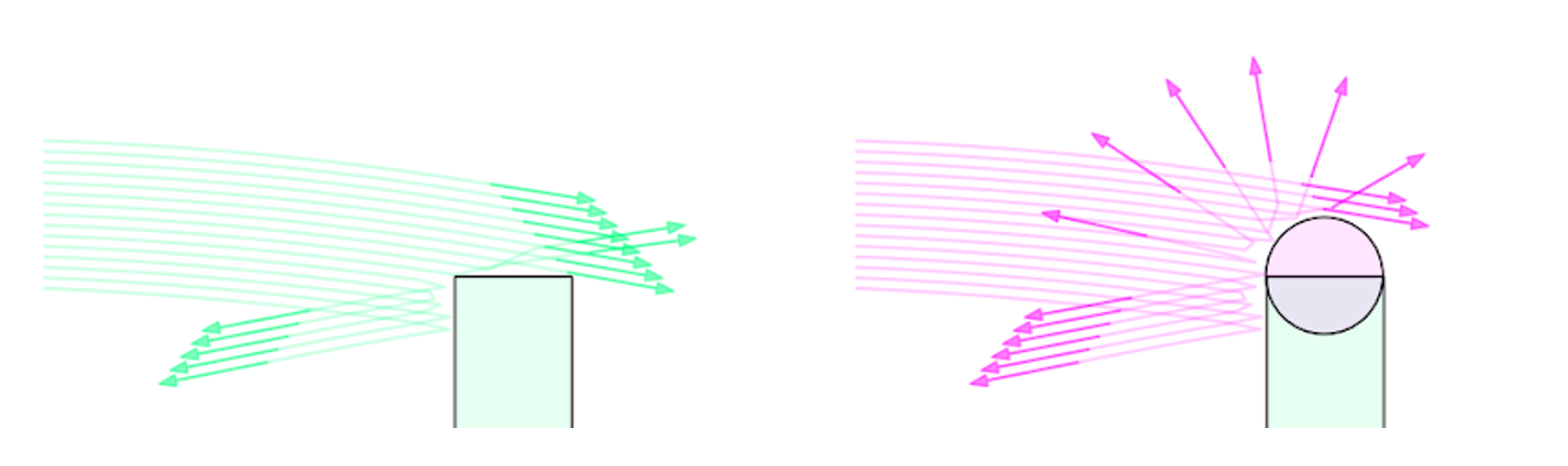}
\caption{Left: More detailed example of ball hitting the wall in \Cref{fig:systems}.B. Left: The green trajectories hit a rectangular wall, displaying discontinuities. Right: the pink trajectories collide with the dome on top, and show continuous but stiff behavior.}
\label{fig:geometry}
\end{figure}

Empirical bias naturally arises for discontinuities or stiff continuous approximations. We give two examples of common discontinuities that arise in differentiable simulation, that permit continuous approximations. 
\begin{example}[\textbf{Coulomb friction}]\normalfont\label{exmp:coulomb}
    The Coulomb model of friction is discontinuous in the relative tangential velocity between two bodies. In many simulators \cite{add,tamsi}, it is common to consider a continuous approximation instead. We idealize such approximations through a piecewise linear relaxation of the Heaviside that is continuous, parametrized by the width of the middle linear region $\nu$ (which corresponds to \emph{slip tolerance}).
    \begin{align*}
        \bar{H}_\nu(t) = \begin{cases}
         2t/\nu & \text{ if } |t|\leq \nu/2 \\
         H(t) & \text{ else}
        \end{cases}.
    \end{align*}
    In practice, lower values of $\nu$ lead to more realistic behavior in simulation \cite{drake}, but this has adverse effects for empirical bias. Considering $f_{\nu}(\btheta,\bw) = \bar{H}_\nu(\btheta + \bw)$, we have $F_{\nu}(\btheta) = \Exp_{\bw}[\bar{H}_\nu(\btheta + \bw)] := \erf(\nu/2-\theta;\sigma^2)$. In particular, setting $c_{\sigma} := \frac{1}{\sqrt{2\pi} \sigma}$, then at $\btheta = \nu/2$, $\nabla F_{\nu}(\btheta) = c_{\sigma}$, whereas, with probability at least $c_{\sigma} \nu$,  $\nabla f_{\nu}(\btheta,\bw) = 0$. Hence, the \fobg{} has $(c_{\sigma}\nu,c_{\sigma},0)$ empirical bias, and its variance scales with $1/\nu$ as $\nu \to 0$. The limiting $\nu = 0$ case, corresponding to the Coulomb model, is the Heaviside from \Cref{exmp:heavisied}, where the limit of high empirical bias, as well as variance, becomes biased in expectation (but, surprisingly, zero variance!). We empirically illustrate this effect in Figure \ref{fig:friction}. We also note that more complicated models of friction (e.g. that incorporates the Stribeck effect \cite{stribeck}) would suffer similar problems.
    
    \begin{figure}[t]
    \centering\includegraphics[width = 0.5\textwidth]{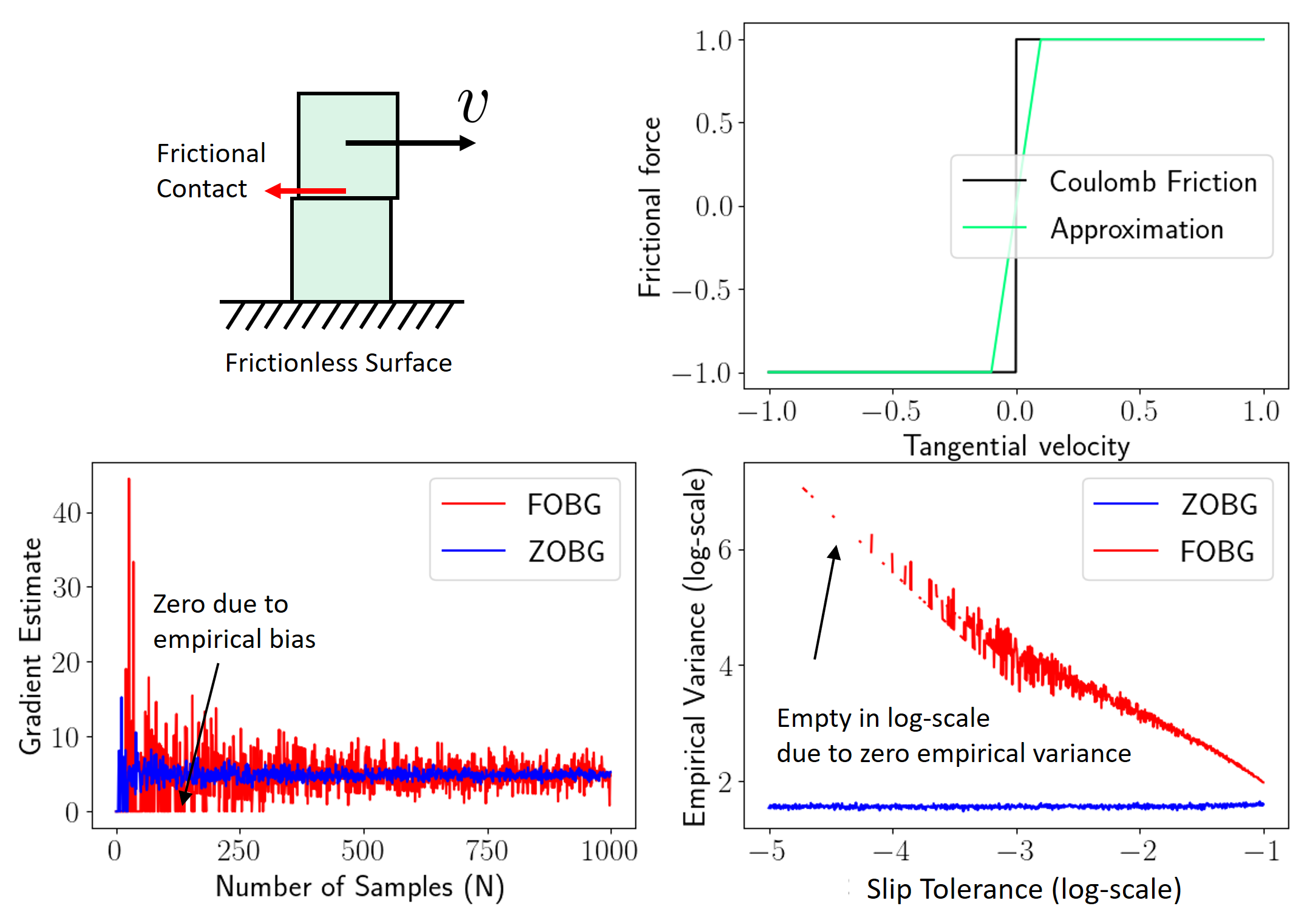}
    \caption{Top column: illustration of the physical system and the relaxation of Coulomb friction. Bottom column: the values of estimators and their empirical variances depending on number of samples and slip tolerance. Values of \fobg{} are zero in low-sample regimes due to empirical bias. As $\nu\rightarrow 0$, the empirical variance of \fobg{} goes to zero, which shows as empty in the log-scale. Expected variance, however, blows up as it scales with $1/\nu$.}
    \label{fig:friction}
    \end{figure}
\end{example}

\begin{figure}[b]
\centering\includegraphics[width = 0.5\textwidth]{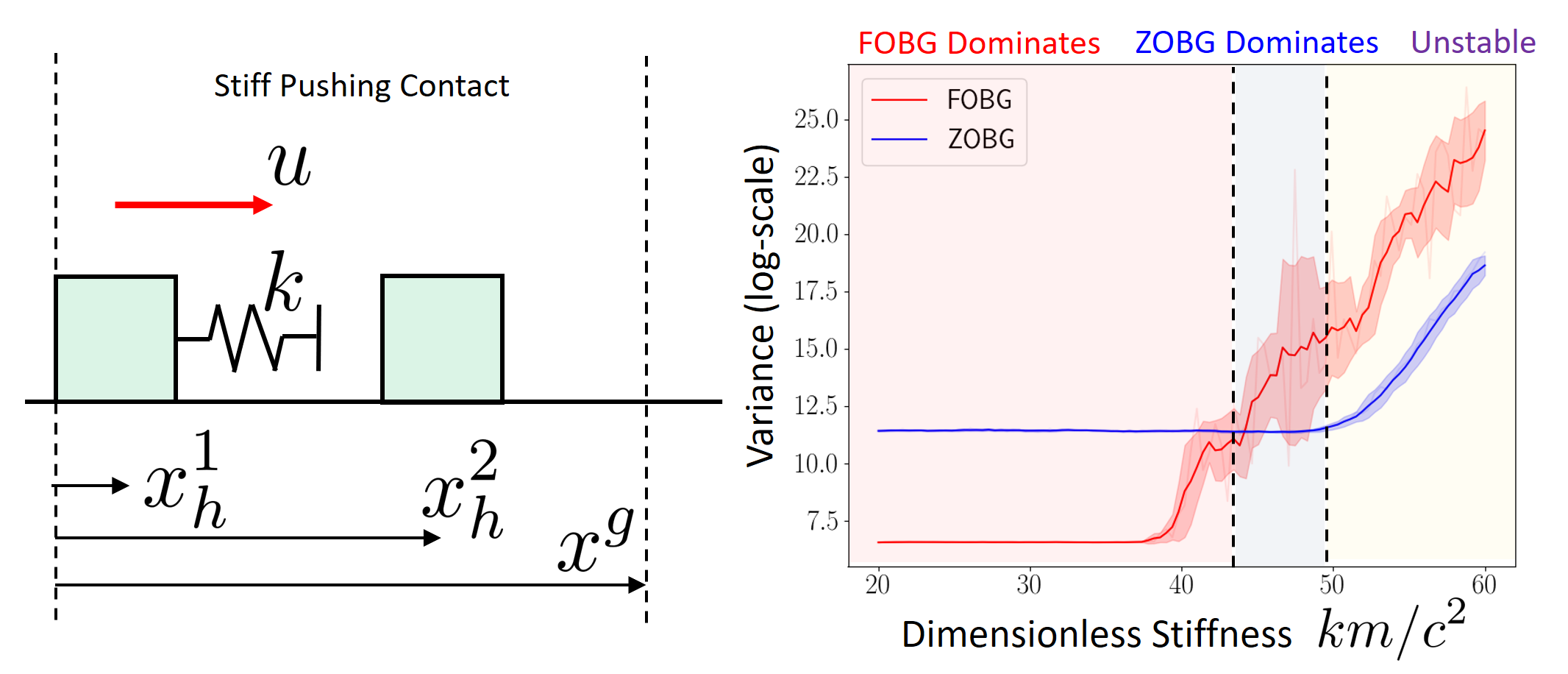}
\caption{The variance of the gradient of $V_1$, with running cost $c_h=\|\bx_h^2-\bx^g\|^2$, with respect to input trajectory as spring constant $k$ increases. Mass $m$ and damping coefficient $c$ are fixed.}
\label{fig:stiff}
\end{figure}

\begin{example}(\textbf{Geometric Discontinuity}). Discontinuity also comes from surface normals. We show this in \Cref{fig:geometry}, where balls that collide with a rectangular geometry create discontinuities. It is possible to make a continuous relaxation \cite{Elandt2019APF} by considering smoother geometry, depicted by the addition of the dome in \Cref{fig:geometry}. While this makes \fobg{} no longer biased asymptotically, the stiffness of the relaxation results in high empirical bias. 
\end{example}
\vskip -0.2 true in
\subsection{High variance first-order estimates}
Even in the absence of empirical bias, we present other cases in which \fobg{} suffers simply due to high variance.

\textbf{Scenario 1: Persistent stiffness.} When the dynamics are \emph{stiff} \footnote{We say that a discrete-time dynamical system is stiff if the mapping function $\phi$ is stiff. Note that when $\phi$ contains forces from spring-like components, $\|\nabla \phi\|$ scales with the spring constant. Thus, the presence of a stiff spring leads to a stiff system.}, such as contact models with stiff spring approximations \cite{huntcrossley}, the high norm of the gradient can contribute to high variance of the \fobg{}.

\begin{example}(\textbf{Pushing with stiff contact}).
We demonstrate this phenomenon through a simple 1D pushing example in \Cref{fig:stiff}, where the \zobg{} has lower variance than the \fobg{} as stiffness increases, until numerical semi-implicit integration becomes unstable under a fixed timestep.

\end{example}

In practice, lowering the timestep can alleviate the issue at the cost of more computation time. Less stiff formulations of contact dynamics \cite{stewarttrinkle,brianmirtich} also addresses this problem effectively.

\textbf{Scenario 2: Chaos.} As noted in \cite{metz2021gradients}, even if the gradient of the dynamics is small at every $h$, their compounding product can cause $\|\nabla_{\btheta} V_1\|$ to be large if the system is chaotic. Yet, in expectation, the gradient of the stochastic objective $\nabla F = \nabla \Exp[V_1]$ can be benign and well-behaved \cite{chaos}.

\begin{example}(\textbf{Chaos of double pendulum}).
We demonstrate this in \Cref{fig:pendulum} for a classic chaotic system of the double pendulum. As the horizon of the trajectory increases, the variance of \fobg{} becomes higher than that of \zobg{}.

\begin{figure}[t]
\centering\includegraphics[width = 0.4\textwidth]{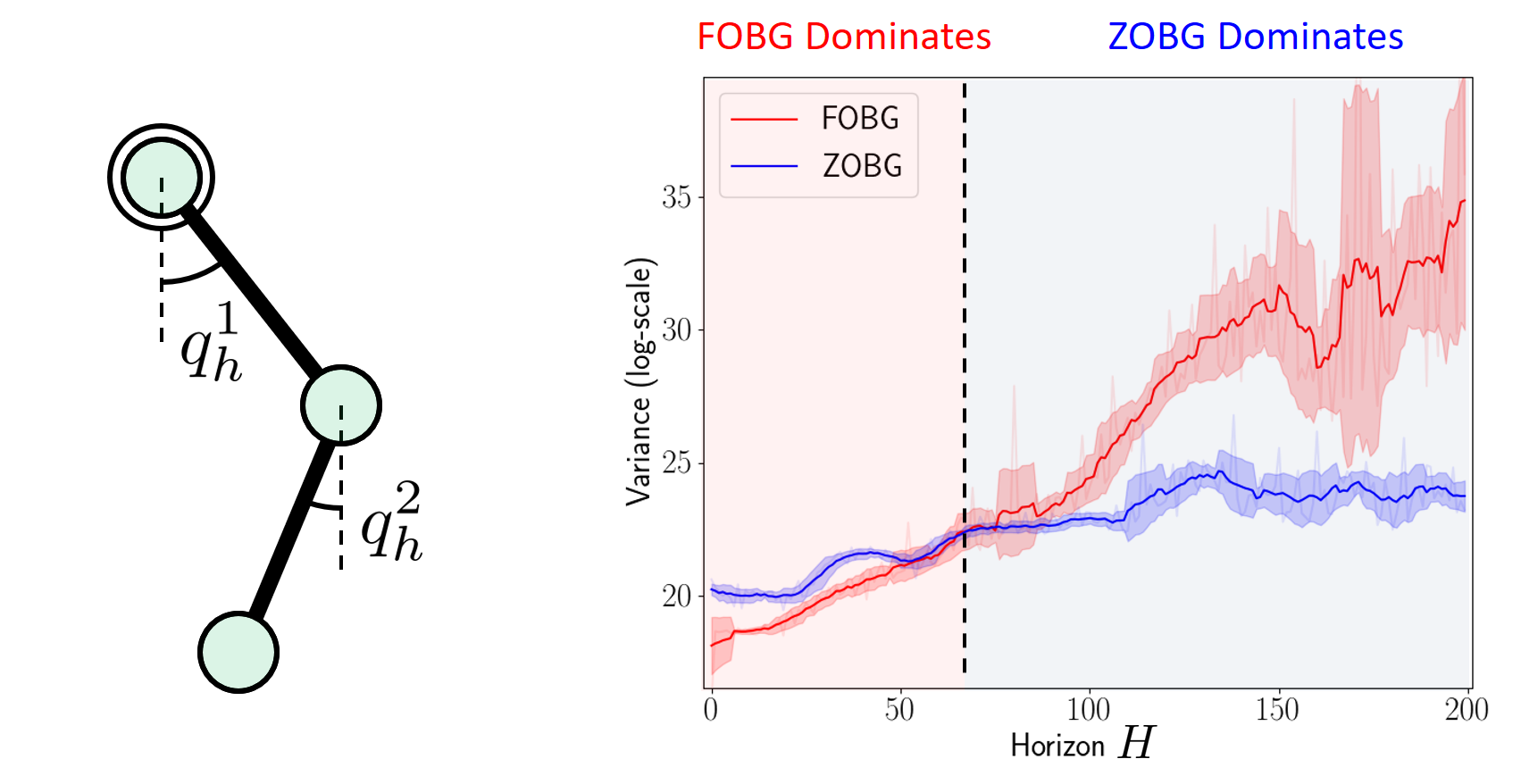}
\caption{Variance of the gradient of the terminal cost $\|q_H - q^g\|^2$ with respect to the initial position $q_1$. As horizon grows through a chaotic system, the \zobg{} dominates the \fobg{}.}
\label{fig:pendulum}
\end{figure}
\end{example} 
\paragraph{Comparison to \zobg{}.}
Compared to the pitfalls of \fobg{}, the \zobg{} variance can be bounded as follows.
\begin{lemma}\label{lem:var_zero_order_one} If for all $\bx$ and $\bar{\bw}$, $|V_1(\bx,\barbw,\btheta)| \le B_V$ and $\|\rmD_{\btheta} \pi(\bx,\btheta)\|_{\mathrm{op}} \le B_{\pi}$, then 
\begin{align*}
\Variance[\nabbarzero F(\btheta)] = \frac{1}{N}\Variance[\nabhatzero F_i(\btheta)] \le \frac{B_V^2B_{\pi}^2}{N}\cdot \frac{Hn}{\sigma^2}.
\end{align*}
\end{lemma}
We refer to \Cref{sec:proof:lem_variance_zero_order} for proof. \Cref{lem:var_zero_order_one} is intended to provide a qualitative understanding of the zeroth-order variance: it scales with the horizon-dimension product $Hn$, but \emph{not} the scale of the derivatives. On the other hand, the variance of \fobg{} does; when  $\frac{Hn}{\sigma^2} \gg \Variance[\nabhatone F(\btheta)] = \Variance[\nabla_{\btheta} V(\bx_1,\barbw,\btheta)]$, the \zobg{} has higher variance.

\begin{figure*}[t]
\centering\includegraphics[width = 1.0\textwidth]{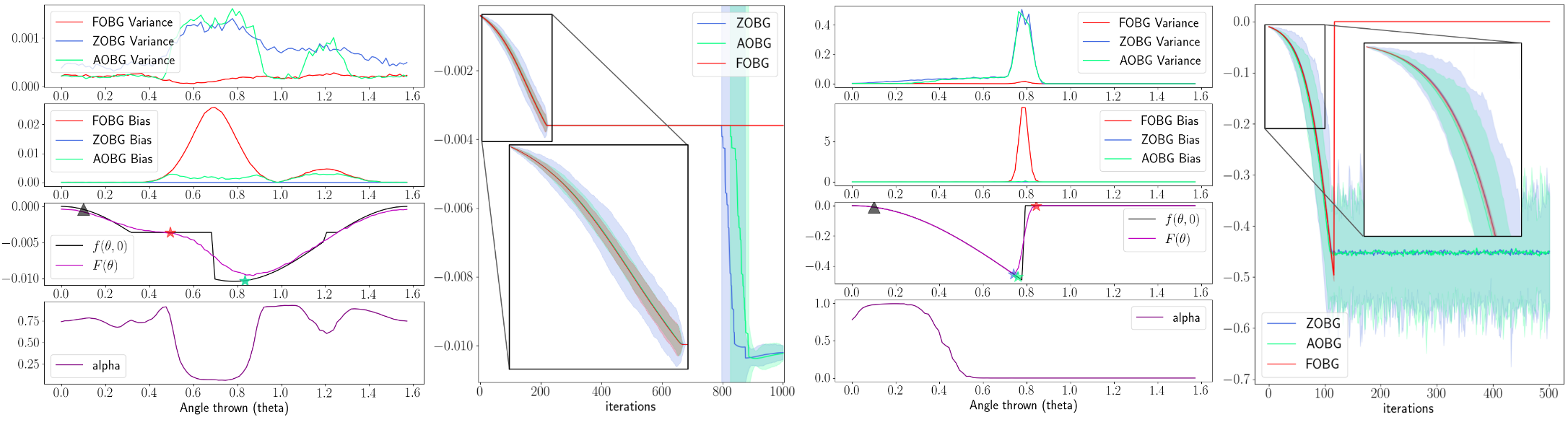}
\vskip -0.1 true in
\caption{First Column: \textbf{Ball with wall} example. In the third row, the triangle is the initial point, and red/blue/green stars are the optimum achieved by \fobg{}, \zobg{}, and AoBG respectively (blue and green stars overlap). Second column: Iteration vs. Cost plot of different gradients. Right columns: Same plot repeated for the \textbf{Momentum Transfer} example. Standard deviation plotted 10 fold for visualization.}
\label{fig:simplesystems}
\end{figure*}

%% file: body/interpolation.tex
\section{$\alpha$-order Gradient Estimator}
\vskip -0.1 true in
\newcommand{\aobg}{\namefont{AoBG}}
\newcommand{\aoeg}{\namefont{AoEG}}

\newcommand{\nabbaralpha}{\nabbar^{[\alpha]}}
\newcommand{\nabalpha}{\nabla^{[\alpha]}}

\newcommand{\nabk}{\nabla^{[k]}}
Previous examples give us insight on which landscapes are better fit for first-order estimates of policy gradient, and which are better fit for zeroth-order ones. As shown in \Cref{fig:simplesystems}, even on a single policy optimization objective, it is best to adaptively switch between the first and zeroth-order estimators depending on the local characteristics of the landscape. In this section, we propose a strategy to achieve this adaptively, interpolating between the two estimators to reap the benefits of both approaches simultaneously.  
\begin{definition} Given $\alpha\in[0,1]$, 
    we define the alpha-order batched gradient (\aobg) as:
    \begin{align*}
        \nabbaralpha F(\btheta) & = \alpha \nabbarone F(\btheta) + (1-\alpha) \nabbarzero F(\btheta).  
    \end{align*}
\end{definition}
When interpolating, we use independent trajectories to generate $\nabbarone F(\btheta)$  and $\nabbarzero F(\btheta)$ (see \Cref{app:interpolatedbiasvariance}). We consider strategies for selecting $\alpha$ in a \emph{local fashion}, as a function of the observed sample, as detailed below. 

\subsection{A robust interpolation protocol}

A potential approach might be to select $\alpha$ based on achieving minimum variance \cite{parmas,pathologies}, considering empirical variance as an estimate. However, in light of the \emph{empirical bias} phenomenon detailed in \Cref{sec:bias_variance} (or even actual bias in the presence of discontinuities), we see that the empirical variance is unreliable, and can lead to inaccurate estimates for our setting. For this reason, we consider an additional criterion of \emph{uniform accuracy}:

\begin{definition}[Accuracy]\label{def:robust}
    $\alpha$ is $(\gamma,\delta)$-accurate if the bound on the \emph{error} of AoBG is satisfied with probability $\delta$:
    \begin{align}\label{eq:robustness}
        \| \nabbaralpha F(\btheta) - \nabla F(\btheta)\| \le \gamma.
    \end{align}
\end{definition}

To remedy the limitations of considering empirical variance in isolation, we propose an interpolation protocol that can satisfy an accuracy guarantee, while still attempting to minimize the variance.
\begin{equation}
\begin{aligned}
        \min_{\alpha \in [0,1]} & \quad \alpha^2 \sighatsq_1 + (1-\alpha)^2 \sighatsq_0 \\
        \text{s.t.} & \quad \epsilon + \alpha \underbrace{\|\nabbarone F - \nabbarzero F\|}_B \leq \gamma.
    \label{eq:tractableform}
\end{aligned}
\end{equation}
We explain the terms in \Cref{eq:tractableform} below in detail.

\textbf{Objective.} Since we interpolate the \fobg{} and \zobg{} using independent samples, $\alpha^2 \sighatsq_1 + (1-\alpha)^2 \sighatsq_0$ is an unbiased estimate of $N\cdot\Variance[\nabbaralpha F(\btheta)]$. Thus, our objective is to choose $\alpha$ to minimize this variance. 

\textbf{Constraint.} Our constraint serves to enforce accuracy. Since the \fobg{} is potentially biased, we use  \zobg{} as a surrogate of $\nabla F(\btheta)$. For this purpose, we use $\epsilon > 0$ as a confidence bound on $\|\nabbarzero F(\btheta) - \nabla F(\btheta)\|$ from the obtained samples. When $\epsilon$ is a valid confidence bound that holds with probability $\delta$, we prove that our constraint in \Cref{eq:tractableform} guarantees accuracy in \Cref{eq:robustness}.

\begin{lemma}[Robustness]\label{lem:robustness} Suppose that $\epsilon + \alpha B \leq \gamma$ with probability $\delta$. Then, $\alpha$ is $(\gamma,\delta)$-accurate.
\end{lemma}
\begin{proof}
By repeated applications of the triangle inequality. See \Cref{app:triangle} for a detailed proof.
\end{proof}

\textbf{Specifying the confidence $\epsilon > 0$.} We select $\epsilon > 0$ based on a Bernstein vector concentration bound (\Cref{app:bernstein}), which only requires a prior upper bound on the magnitude of the value function $V_1(\cdot)$ and gradients $\rmD_{\btheta} \pi(\cdot,\btheta)$.

\textbf{Asymptotic feasibility.} \Cref{eq:tractableform} is not feasible if $\epsilon>\gamma$, which would indicate that we simply do not have enough samples to guarantee $(\gamma,\delta)$-accuracy. In this case, we choose to side on conservatism and fully use the \zobg{} by setting $\alpha=0$. Asymptotically, as the number of samples $N\rightarrow \infty$, the confidence interval $\varepsilon\rightarrow 0$, which implies that \Cref{eq:tractableform} will always be feasible. 

Finally, we note that \Cref{eq:tractableform} has a closed form solution, whose proof is provided in \Cref{app:closed_form}.
\begin{lemma}\label{lem:closed_form}
    With $\gamma = \infty$, the optimal $\alpha$ is $\alpha_{\infty} := \frac{\hat{\sigma}_0^2}{\hat{\sigma}_1^2 + \hat{\sigma}_0^2}$. For finite $\gamma \ge \epsilon$, \Cref{eq:tractableform} is 
    \begin{equation}
        \alpha_{\gamma} := \begin{cases}
            \alpha_{\infty} & \text{ if } \quad \alpha_{\infty} B \leq \gamma - \varepsilon \\ 
            \frac{\gamma-\varepsilon}{B} & \text{ otherwise } .
        \end{cases}
    \end{equation}
\end{lemma}

We give some qualitative characteristics of the solution:
\begin{itemize}\itemsep -0.2em
\item If we are within constraint and $\hat{\sigma}_0^2\gg\hat{\sigma}_1^2$, as we can expect from benign smooth systems, then $\alpha\approx 1$, and we rely more on the \fobg{}.
\item In pathological cases where we are unbiased yet $\hat{\sigma}_1^2 \gg \hat{\sigma}_0^2$ (e.g. stiffness and chaos), then $\alpha\approx0$.
\item If there is a large difference between the \zobg{} and the \fobg{} such that $B\gg 0$, we expect strict/empirical bias from discontinuities and tend towards using \zobg{}. 
\end{itemize}

%% file: body/case_studies.tex

\begin{figure*}[t]
\centering\includegraphics[width = 1.0\textwidth]{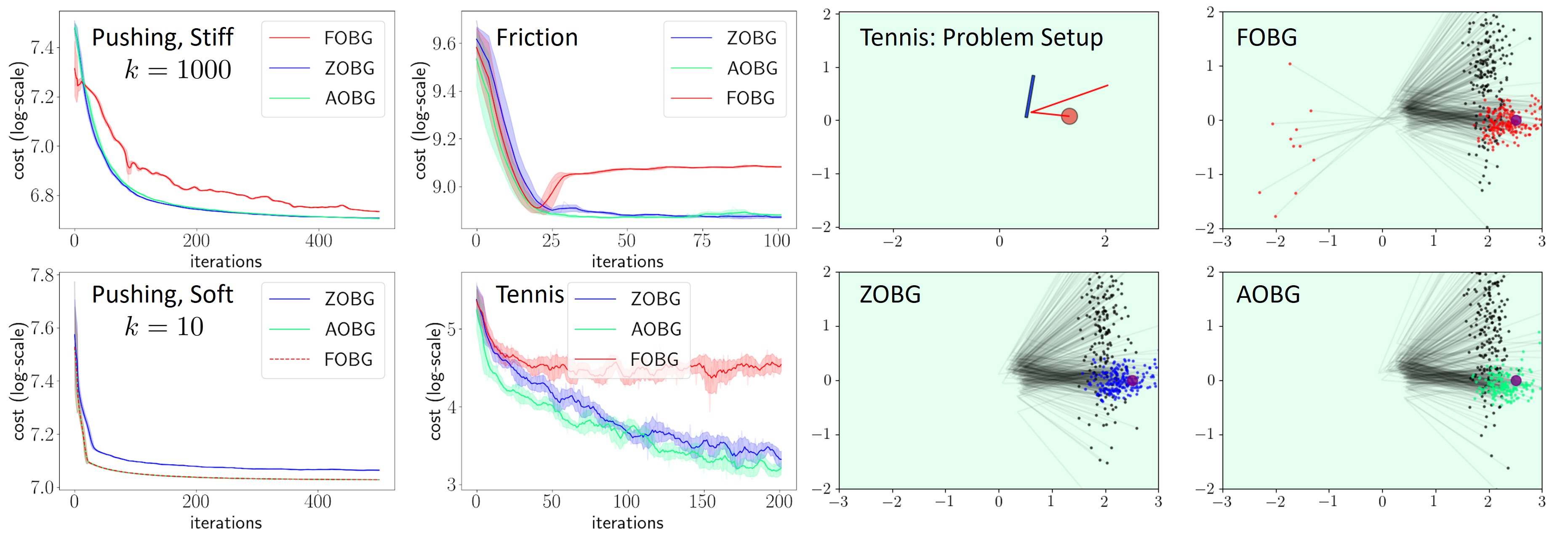}
\vskip -0.1 true in
\caption{1st column: trajectory optimization on pushing example with different contact models. \aobg{} and \fobg{} overlaps in soft pushing example. 2nd column: trajectory optimization on friction contact, and policy optimization on the tennis example. 3rd / 4th column: Visualization of policy performance for tennis. Black dots correspond to initial positions and colored dots correspond to final position, while the shaded lines are visualizations of individual closed-loop trajectories across multiple initial conditions.}
\label{fig:results}
\end{figure*}

\section{Landscape Analysis \& Case Studies}
\subsection{Landscape analysis on  examples}
Though we have characterized the bias-variance characteristics of different gradients, their convergence properties in landscapes of physical systems remain to be investigated. We visualize the performance of fixed-step gradient descent with the \fobg{}, \zobg{}, and AoBG on examples of Figure \ref{fig:systems}.

\textbf{Ball with wall.} On the system of \Cref{fig:systems}.B, the FoBG fails to make progress at the region of flatness, while the \zobg{} and AoBG successfully find the minima of the landscape (\Cref{fig:simplesystems}). In addition, the interpolation scheme switches to prioritizing ZoBG near discontinuities, while using more information from \fobg{} far from discontinuities; as a result, the variance of AoBG is lower than that of \zobg{}.

\textbf{Angular momentum transfer.} Next, we show results for the momentum transfer system of \Cref{fig:systems}.C in Figure \ref{fig:simplesystems}. Running gradient descent results in both the ZoBG and AoBG converging to the robust local minima of the solution. However, the bias of FoBG forces it off the cliff and the optimizer is unable to recover. Again, our interpolation scheme smoothly switches to prioritizing the ZoBG near the discontinuity, enabling it to stay within the safe region while maximizing the transferred momentum. 

\textbf{Bias-variance leads to different minima.} Through these examples with discontinuities, we claim that the bias-variance characteristics of gradients in these landscapes not only lead to different convergence rates, but convergence to different minima. The same argument holds for \emph{nearly discontinuous} landscapes that display high empirical bias. Both estimators are unbiased in expectation, and the high variance of \fobg{} should manifest itself in worse convergence rates. Yet, the high \emph{empirical bias} in the finite-sample regime leads to low empirical variance and different minima, leading to performance that is indistinguishable from when the underlying landscape is truly discontinuous.

Combined with the benefits of stochasticity in \Cref{sec:introduction}, we believe that this might explain why zero-order methods in RL are solving problems for physical systems where deterministic (even \emph{stochastic}) first order methods have struggled.

\subsection{Policy optimization case studies}
\vskip -0.1 true in
To validate our results on policy optimization problems with differentiable simulators, we compare the performance of different gradients on time-stepping simulations written in \texttt{torch} \cite{pytorch}. For all of our examples, we validate the correctness of the analytic gradients by comparing the values of \fobg{} and \zobg{} on a one-step cost. 

To empirically verify the various hypotheses made in this paper, we compare the performance of three gradient estimators: the \fobg{} and \zobg{}, which uniformly utilizes first and zeroth-order gradients, and the AoBG, which utilizes our robust interpolation protocol.

\textbf{Pushing: Trajectory optimization.}  We describe performance of gradients on the pushing (\Cref{fig:stiff}) environment, where contact is modeled using the \emph{penalty method} (i.e. stiff spring) with additional viscous damping on the velocities of the system. We use horizon of $H=200$ to find the optimal force sequence of the first block to minimize distance between the second block and the goal position. Our results in Figure \ref{fig:results} show that for soft springs $(k=10)$, the \fobg{} outperforms the \zobg{}, but stiffer springs $(k=1000)$ results in the \zobg{} outperforming the \fobg{}. This confirms our hypothesis that the stiffness of contact models has direct correlations with the variance of the estimators, which in turn affects the convergence rate of optimization algorithms that use such estimators.

In addition, we note that the interpolated gradient AoBG is able to automatically choose between the two gradients that performs better by utilizing empirical variance as a statistical measure of performance. 

\textbf{Friction: Trajectory Optimization}. We describe performance of gradients on the friction (\Cref{fig:friction}) environment.

Although the \fobg{} initially converges faster in this environment, it is unaware of the discontinuity that occurs when it slides off the box. As a result, the performance quickly degrades after few iterations. On the other hand, the \aobg{} and \zobg{} successfully optimize the trajectory, with \aobg{} showing slightly faster convergence.

\textbf{Tennis: Policy optimization.} Next, we describe the performance of different gradients on a tennis environment (similar to breakout), where the paddle needs to bounce the ball to some desired target location. We use a linear feedback policy with $d=21$ parameters, and horizon of $H=200$. In order to correctly obtain analytic gradients, we use continuous event detection with the time of impact formulation \cite{difftaichi}. The results of running policy optimization is presented in \Cref{fig:results}. 

While the \zobg{} and the AoBG are successful in finding a policy that bounces the balls through different initial conditions, the \fobg{} suffers from the discontinuities of geometry, and still misses many of the balls. Furthermore, the AoBG still converges slightly faster than the \zobg{} by utilizing first-order information.

%% file: body/discussion.tex

\section{Discussion}
In this section, we elaborate and discuss on some of the ramifications of our work.

\textbf{Impact on Computation Time.} The convergence rate of gradient descent in stochastic optimization scales directly with the variance of the estimator \cite{ghadimilan}. For smooth and well-behaved landscapes, \fobg{} often converges faster since $\textbf{Var}[\nabbarone F] < \textbf{Var}[{\nabbarzero F}]$. However, when there are discontinuities or near-discontinuities in the landscape, this promise no longer holds since gradient descent using \fobg{} might not converge due to bias. Indeed, \Cref{exmp:coulomb} tells us that bias due to discontinuities can be interpreted as infinite variance. Under this interpretation, the convergence rate of gradient descent is ill-defined.

In practice; however, the computation cost of obtaining the gradient must be taken into consideration as well. Given the same number of samples $N$, the computation of \fobg{} is more costly than the \zobg{}, as \fobg{} requires automatic differentiation through the computation graph while \zobg{} simply requires evaluation. Thus, the benefits of convergence rates using the \fobg{} must justify the additional cost of computing them. 

\textbf{Implicit time-stepping.} In our work, we have mainly addressed two classes of simulation methods for contact. The first uses the penalty method \cite{add,drake}, which approximates contact via stiff springs, and the second uses event detection \cite{difftaichi}, which explicitly computes time-of-impact for automatic differentiation.

In addition to the ones covered, we note a third class of simulators that rely on optimization-based implicit time-stepping \cite{mujoco,bullet,flex,pangsim,dojo}, which can be made differentiable by sensitivity analysis \cite{boyd}. These simulators suffer less from stiffness by considering more long-term behavior across each timestep; however, geometrical discontinuities can still remain problematic. We leave detailed empirical study using these simulators to future work.

\textbf{Analytic Smoothing.} Randomized smoothing relies on smoothing out the policy objective via the process of noise injection and sampling. However, one can also resort to \emph{analytic smoothing}, which finds analytically smooth approximation of the underlying dynamics $\phi$ \cite{plasticinelab,dojo}. Modifying and smoothing $\phi$ directly also has the effect of smoothing the induced value function, though the resulting landscape will be different from the landscaped induced by appending noise to the policy output. 

However, even when $\phi$ can be analytically smoothed, Monte-Carlo sampling is still required for optimization across initial conditions $\rho$. For such settings, the findings of \Cref{sec:empiricalbias} is still highly relevant, as the performance of \fobg{} still suffers from the stiffness of the smoothed approximation of $\phi$. However, as many smoothing methods provide access to parameters that control the strength of smoothing, algorithms may be able to take a curriculum-learning approach where dynamics become more realistic, and less smooth, as more iterations are taken for policy search.

%% file: body/conclusion.tex
\vspace{-1em}
\section{Conclusion} 
\label{sec:conclusion}
%
\vspace{-1em}
Do differentiable simulators give better policy gradients? We have shown that the answer depends intricately on the underlying characteristics of the physical systems. While Lipschitz continuous systems with reasonably bounded gradients may enjoy fast convergence given by the low variance of first-order estimators, using the gradients of differentiable simulators may \emph{hurt} for problems that involve nearly/strictly discontinuous landscapes, stiff dynamics, or chaotic systems. Moreover, due to the empirical bias phenomenon, bias of first-order estimators in nearly/strictly discontinuous landscapes cannot be diagnosed from empirical variance alone. We believe that many challenging tasks that both RL and differentiable simulators try to address necessarily involve dealing with physical systems with such characteristics, such as those that are rich with contact.

These limitations of using differentiable simulators for planning and control need to be addressed from both the design of simulator and algorithms: from the simulator side, we have shown that certain modeling decisions such as stiffness of contact dynamics can have significant  underlying consequences in the performance of policy optimization that uses gradients from these simulators. From the algorithm side, we have shown we can automate the procedure of deciding which one to use online via interpolation.

\section*{Acknowledgements}
This work was funded by Amazon PO 2D-06310236, Lincoln Laboratory/Air Force Contract No. FA8702-15-D-0001, Defense Science \& Technology Agency No.DST00OECI20300823, NSF Award EFMA-1830901, and the Ocado Group. We would also like to thank the anonymous ICML reviewers for their valuable feedback on the manuscript. 

%% file: appendix/max_derivations.tex

\onecolumn

~\\
\centerline{{\fontsize{13.5}{13.5}\selectfont \textbf{Supplementary Materials for}}}

\vspace{6pt}
\centerline{\fontsize{13.5}{13.5}\selectfont \textbf{
	  ``Do Differentiable Simulators Give Better Policy Gradients''}}
\vspace{10pt}   

\fontsize{11}{13.2}\selectfont

\newcommand{\cZ}{\mathcal{Z}}
\newcommand{\cD}{\mathcal{D}}
\newcommand{\bL}{\mathbf{L}}
\newcommand{\bh}{\mathbf{h}}
\newcommand{\cC}{\mathcal{C}}
\newcommand{\bv}{\mathbf{v}}
\newcommand{\cN}{\mathcal{N}}

\section{Formal Expected Gradient Computations}\label{app:formal}

This section establishes rigorous unbiasedness guarantees for the \zobg{} (under general conditions) and of the \fobg{} (under more restrictive conditions). Specifically, \Cref{cor:gaussian_zobg} provides a rigorous version of the \zobg{} guarantee, \Cref{lem:ZOBG_bias}, which is a special case of \Cref{prop:reinforce_rl} which holds for general, possibly non-Gaussian noise distributions. The \fobg{} estimator is addressed in \Cref{prop:lipschitz_rl}, which provides the rigorous statement of \Cref{lem:FOBG_bias}. We present a lengthy preliminaries section, \Cref{app:formal_prelim}, to formalize the results that follow. We then follow with formal statements of the results, \Cref{app:formal_results_things}, and defer proofs to \Cref{app:proofs_for_unbiasedness}. The preliminaries below are requisites only for the results and proofs within this section, and are not needed in future appendices. 

\subsection{Preliminaries}\label{app:formal_prelim}
Throughout, $\|\cdot\|$ denotes the Euclidean norm of vectors. We begin by specifying our sense of expectations and derivatives, and then turn to other, less-standard preliminaries. To rigorously describe expectations of non-continuous functions and of derivatives of non-smooth functions, we start with some preliminaries from measure theory.  

\newcommand{\cX}{\mathcal{X}}
\newcommand{\Borel}[1][d]{\mathscr{B}(\R^{#1})}
\newcommand{\Lebsig}[1][d]{\mathscr{Led}(\R^{#1})}

\newcommand{\scrF}{\mathscr{F}}
\newcommand{\ba}{\mathbf{a}}
\newcommand{\Lebz}{\mathscr{L}(\cZ)}
\paragraph{Lebesgue measurability. } For a background on measure theory, we direct the reader to \cite{stein2009real}. Here, we recall a few definitions. We define the set of Lebesgue measurable sets $\Lebsig[D]$ as the collection of subset $\cZ \subset \R^D$ for which the Lebesgue measure is well-defined. We let $\Borel[D'] \subset \Lebsig[D]$ be the collection of Borel measurable sets on $\R^{D'}$. We say a mapping $\Phi: \R^D \to \R^{D'}$ is \emph{Lebesgue measurable} if for all $\cZ' \in \Borel[D']$, $\Phi^{-1}(\cZ') \in \Lebsig[D]$. We say it is \emph{Borel measurable} if, more strongly, it holds that $\Phi^{-1}(\cZ') \in \Borel[D]$. The composition of Borel measurable functions are Borel measurable, but the same is not true more generally for Lebesgue measurable functions. Throughout, all functions are assumed Borel measurable unless otherwise specified, so their compositions are also Borel measurable.

More generally, given a Lebesgue measurable set $\cZ \subset \R^D$, we define $\Lebz$ as the set $\{\cZ \cap \tilde{\cZ} :\tilde{\cZ} \in \Lebsig[R^d] \}$, and say a function $\Phi: \cZ \to \R^{D'}$ is Lebesgue mearuable on its domain if for all $\cZ' \in \Borel[D']$, $\Phi^{-1}(\cZ') \in \Lebsig[D]$.

\paragraph{Lebesgue complete distribution.} We consider probability distributions $\cD$ on $\R^{D}$ which assign probability to \emph{all Lebesgue measurable sets} $\cZ \subset \R^D$: i.e., $\Pr_{\bz \sim \cD}[\bz \in\cZ]$ is well defined. Note that these distribution do not need to have density with respect to the Lebesgue measure: indeed, continuous, discrete, and mixture of continuous and discrete distributions all can be defined to assign probabilities to all Lebesgue-measurable sets.

We say $\cZ \subset \R^D$ is \emph{$\cD$-null} if $\Pr_{\bz \sim \cD}[\bz \in\cZ] = 0$. We assume without loss of generality that $\cD$ is \emph{complete}, so that given a $\cD$-null Lebesgue measurable set $\cZ$, $\Pr_{\bz \sim \cD}[\bz \in\cZ'] $ is well defined and equal to zero for all $\cZ' \subset \cZ$. 
We call distributions which are complete and assign probability to all Lebesgue sets \emph{Lebesgue complete}. We shall assume without comment that all distributions are Lebesgue complete.

\paragraph{Almost-everywhere functions and expectation.} 
Given a Lebesgue complete distribution $\cD$ on $\R^D$, we define expectation of a Lebesgue measurable $\Phi: \R^{D}\to \R^{D'}$ in the standard way. We say a function $\Phi$ is defined \emph{$\cD$-almost-surely} if there exists a Lebesgue-measurable set $\cZ \subset \R^D$ such that $\Phi$ is a Lebesgue measurable as mapping $\cZ \to \R^{D'}$, and $\cZ^c = \R^D \setminus \cZ$ is $\cD$-null. Given such a function $\Phi$, we define its expectation
\begin{align}
\Exp_{\bz \sim \cD}[\Phi(\bz)] := \Exp_{\bz \sim \cD}[\tilde\Phi(\bz)], \quad \text{ where }
\tilde{\Phi}(\bz) = \begin{cases} \Phi(\bz) &\bz  \in \cZ \\
0 & \text{otherwise.} \label{exp:ae}
\end{cases}
\end{align}
One can verify that $\tilde{\Phi}(\bz)$ is Lebesgue measurable.  Note that this definition is independent of the choice of $\cZ$: if $\cZ'$ is another set witnessing the almost-sure definition of $\Phi$, then the induced map $\tilde{\Phi}'$ defined by applying \Cref{exp:ae} with $\cZ'$ is also Lebesgue measurable, $\tilde{\Phi}' = \tilde{\Phi}$ $\cD$-almost surely, so that the integrals coincide. 
\begin{example}[Heaviside, revisited] With definition in \Cref{exp:ae}, we see that the derivative of the example in \Cref{exmp:heavisied} is $0$ almost surely under $\bw \sim p$; that is, the event on which the derivative of the Heaviside is both undefined has probability zero when $\bw \sim p$, and outside this event, its derivative is identically zero. 
\end{example} 

Stated simply, we ignore values of $\Phi$ defined outside the probability-one set $\cZ$. This definition has numerous advantages. For one, it satisfies the law of large numbers. That is, 
\begin{itemize}
\item If $\Exp_{\bz \sim \cD}\|\tilde{\Phi}(\bz)\| < \infty$, then for $\bz^{(1)},\dots,\bz^{(N)} \iidsim \cD$,  $\frac{1}{N}\sum_{i=1}^N \Phi(\bz^{(i)}) $ converges to $\Exp[\Phi(\bz)]$ in probability. 
\item If $\Exp_{\bz \sim \cD}\|\tilde{\Phi}(\bz)\|^2 < \infty$, this convergence holds almost surely. 
\end{itemize}
For further discussion, we direct the readers to a standard reference on probability theory (e.g. \cite{cinlar2011probability}). 

\paragraph{Multivariable derivative.} We provide conditions under which the multivariable function $F(\btheta):\R^d \to \R$ is differentiable. Formally, we say that a function $\Phi: \R^{d_1} \to \R^{d_1}$ is \emph{differentiable} at a point $\bz \in \R^d$ if there exists a linear map $\rmD \Phi(\bz) \in \R^{d_2 \times d_1}$ such that
\begin{align*}
\lim_{\|\bh\| \to 0} \left\|\frac{\Phi(\bh+\bz) - \Phi(\bz)}{\|\bh\|} - \rmD \Phi(\bz) \cdot \bh \right\| = 0.
\end{align*}
The limit is defined in the sense of $\lim_{\|\bh\| \to 0}(\cdot) = \lim_{t \to 0} \sup_{\|\bh\| \le t}(\cdot)$. Existence of a multivariable derivative slightly stronger that $\Phi(\cdot)$ having directional derivatives, and in particular, stronger than the existence of a gradient (see \citep[Chapter 9]{rudin1964principles} for reference).

\paragraph{Finite moments and polynomial growth. } To ensure all expectations are defined, we consider distributions for which all moments are finite. 
\begin{definition} We say that a (Lebesgue complete) distribution $\rho$ over a random variable $\bz$ \emph{has finite moments} if $\Exp_{\bz \sim \rho}\|\bz\|^a < \infty$ for all $a > 0$.
\end{definition}
The class of function which have finite expectations under distributions with finite moments are functions which have polynomial growth, in the following sense. 
\begin{definition}We say that a function $\psi: \R^{d_1} \to \R^{d_2}$ has \emph{polynomial growth} if there exists constants $a,b > 0$ such that $\|\psi(\bx)\| \le a (1 + \|\bz\|^b)$ for all $\bz \in \R^d$. We say that a matrix (or tensor) valued function has polynomial growth if the vector-valued function corresponding to flattening its entries into a vector has polynomial growth (for matrices, this means $\|\psi(\bz)\|_{\fro} \le a (1 + \|\bz\|^b)$). 
\end{definition}
The following lemma is clear. 
\begin{lemma}\label{lem:poly_growth} Suppose $\rho$ is a distribution over variables $\bx$ which has finite moments, and suppose $g(\bx)$ has polynomial growth. Then $\Exp[g(\bx)]$ is well defined.
\end{lemma}
A second useful (and straightforward to check) fact is that polynomial growth is preserved under marginalization. 
\begin{lemma}\label{lem:poly_growth_marginalization} Suppose $\rho$ is a distribution over variables $\bx$ which has finite moments, and suppose $g(\bz,\bx)$ has polynomial growth in its argument $(\bz,\bx)$. Then $\bz \mapsto \Exp[g(\bz,\bx)]$ is well defined and has polynomial growth in $\bz$.
\end{lemma}
\newcommand{\cU}{\mathcal{U}}
\newcommand{\cB}{\mathcal{B}}
\paragraph{Lipschitz functions.} To establish the unbiasedness of the \fobg{} for non-smooth functions, we invoke the Lipschitz continuity assumption. We say a function $\Phi: \R^{D} \to \R^{D'}$ is locally-Lipschitz if, for every $\bz \in \R^D$, there is a neighborhood a neighborhood $\cU$ of $\bz$ such that there exists an $L > 0$ such that for all $\bz',\bz'' \in \cU$, $\|\Phi(\bz') - \Phi(\bz'')\| \le L\|\bz'-\bz''\|$. Locally Lipschitz functions are continuous, and thus Borel measurable.  
\begin{lemma}[Rademacher's Theorem]\label{lem:Rademacher}Every locally Lipschitz function $\psi:\R^{D} \to \R$ is differentiable on a set of $\cZ \subset \R^D$ such that $\cZ^c = \R^D \setminus \cZ$ has Lebesgue measure zero. 
\end{lemma}
The above result is standard (see, e.g. \citet[Chapter 2]{ern2013theory}).

To ensure convergence of integrals, we consider functions  where the Lipschitz constant grows polynomially in the radius of the domain. 
\begin{definition}[Polynomially Lipschitz] We say that 
\begin{itemize}
\item A function $\psi(\bz):\R^{d_1} \to \R^{d_2}$ is  \emph{polynomially-Lipschitz} if there are constants $a,b > 0$ such that for all radii $R \ge 1$ and all $\bz,\bz' \in \R^{d_1}$ such that $\|\bz\|,\|\bz'\| \le R$, $\|\psi(\bz) - \psi(\bz')\| \le a R^b$. 
\item We say a function $\psi(\bz;\bx):\R^{d_1} \times \R^n \to  \R^{d_2}$ is \emph{parametrized-polynomially-Lipschitz}  if for all radii $R \ge 1$ and all $\bz,\bz' \in \R^{d_1}$ and $\bx \in \R^n$ such that $\|\bz\|,\|\bz'\|,\|\bx\| \le R$, $\|\psi(\bz;\bx) - \psi(\bz';\bx)\| \le a R^b$. 
\end{itemize}
\end{definition}
One can check that polynomially Lipschitz functions are locally Lipschitz.

\subsection{Formal results}\label{app:formal_results_things}

We now state our formal results. Throughout, our smoothing noise $w$ has distribution $p$ which has the following form.
\begin{assumption}\label{asm:w_assumptions} The distribution $p$ admits a density $p(\bw) = e^{\alpha - \psi(\bw)}$, where
\begin{itemize} 
    \item[(a)] $\psi(\bw) \ge a \|\bw\| - b$ for some constants $a > 0$ and $b \in \R$. 
    \item[(b)] $\psi$ is twice differentiable everywhere, and  $\nablatwo \psi(\bw)$ has polynomial growth. 
\end{itemize}
\end{assumption}
\begin{example}[Gaussian distribution] The cannonical example is the Gaussian distribution $\bw \sim \cN(0,\sigma^2 I_n)$, where $p(\bw) = \frac{1}{\sqrt{2\pi} \sigma}\exp(\frac{-\|\bw\|^2}{2\sigma^2})$. Here, $\psi(\bw)=\frac{\|\bw\|^2}{2\sigma^2}$, which has polynomial growth and, being quadratic, is twice differentiable. In addition, 
\begin{align}
\nabla \psi(\bw) = \frac{\bw}{\sigma^2}, \quad \Exp[\nabla  \psi(\bw)] = 0. \label{eq:nab_psi_Gauss}
\end{align}
\end{example}
\paragraph{Zeroth-order unbiasedness.} We now stipulate highly general conditions under which the zeroth-order estimator is unbiased. In the interest of generality, we allow time-varying policies and costs.
\begin{definition}We say that a tuple $(\rho,p,\phi,c_{1:H};\pi_{1:H})$ is a \emph{benign planning problem} if (a) $\rho$ has finite moments (b) $p$ satisfies \Cref{asm:w_assumptions}, (c) the dynamics $\phi(\cdot,\cdot)$ and costs $c_h(\cdot,\cdot)$ have polynomial growth (for all $h \in H$), and  (d), for each $\bx \in \R^n$ and $h \in [H]$, $\bu \mapsto \pi_h(\bx,\bu)$ is twice-differentiable in $\bu$ and its second-order derivative has polynomial growth in $\bx$. In addition, we assume $\phi,c_{1:H},\pi_{1:H}$ are all Borel measurable. 
\end{definition}
We consider the resulting stochastic optimization objective. 
\begin{align*}
F(\btheta) &= \Exp_{\bx_1 \sim \rho, \bw_{1:H} \sim p^H}\left[V_1(\bx_{h},\bw_{1:H},\btheta)\right]\\
&\quad \text{s.t. } \bx_{h+1} = \phi(\bx_h,\bu_h), \quad \bu_h = \pi(\bx_h,\btheta) + \bw_h.
\end{align*}
Note that we define the expectation jointly over $\Exp_{\bx_1 \sim \rho, \bw_{1:H} \sim p^H}$, so as not to assume Fubini's theorem holds (even though, under our assumptions, it does). Our first result is a rigorous statement of the unbiasedness of the zeroth-order estimator. 
\begin{proposition}\label{prop:reinforce_rl} Suppose that $(\rho,p,\phi,c_{1:H};\pi_{1:H})$ is a benign planning problem. Then, the objective $F(\btheta)$ defined in \Cref{eq:F_theta_gen} is differentiable, and 
\begin{align*}
\nabthet F(\btheta) &= \Exp_{\bx_1 \sim \rho,\bw_{1:H} \sim p^H}\left[ \sum_{h=1}^H (\rmD_{\btheta} \pi_h(\bx_{h},\btheta))^\intercal  \psi(\bw_h) V_h(\bx_{h},\bw_{h:H},\btheta)\right].
\end{align*}
If, in addition $\Exp_{\bw \sim p}[\nabla \psi(\bw)] = 0$, we also have
\begin{align*}
\nabthet F(\btheta) &= \Exp_{\bx_1 \sim \rho,\bw_{1:H} \sim p^H}\left[ V_1(\bx_{h},\bw_{1:H},\btheta) \sum_{h=1}^H (\rmD_{\btheta} \pi_h(\bx_{h},\btheta))^\intercal  \psi(\bw_h) \right].
\end{align*}
\end{proposition}
\Cref{eq:nab_psi_Gauss} yields the following corollary for Gaussian distributions, which recovers \Cref{lem:ZOBG_bias} in the main text. 
\begin{corollary}\label{cor:gaussian_zobg} In the special case where $p = \cN(0,\sigma^2 I)$, we have 
\begin{align*}
\nabthet F(\btheta) &= \frac{1}{\sigma^2}\Exp_{\bx_1 \sim \rho,\bw_{1:H} \sim p^H}\left[ V_1(\bx_{h},\bw_{1:H},\btheta) \sum_{h=1}^H (\rmD_{\btheta} \pi_h(\bx_{h},\btheta))^\intercal  \bw_h \right].
\end{align*}
\end{corollary}

\paragraph{First-order unbiasedness under Lipschitzness.} Next, we turn to the formal result under Lipschitzness. We consider objectives which have the following additional assumptions:
\begin{definition}We say that a tuple $(\rho,p,\phi,c_{1:H};\pi_{1:H})$ is a \emph{benign Lipschitz planning problem} if it is a benign planning problem, and in addition, (a) $c_{h}$  and $\pi_h$ are everywhere-differentiable and their derivatives have polynomial growth, and (b) $\phi$ is polynomially Lipschitz.
\end{definition}
In addition, we require one more technical condition which ensures measurability of the set on which the analytic gradients are defined.
\begin{definition} We say that the distribution $\rho$ is \emph{decomposable} if there exists a Lebesgue-measurable function $\mu: \R^n \to \R_{\ge 0}$ and a countable set of atoms $\ba_1,\ba_2,\dots,$ with weights $\nu_1,\nu_2,\dots$ such that, for any $\cX \subset \R^n$, $\Pr_{\bx_1 \sim \rho}[\bx_1 \in \cX] = \int_{\cX} \mu(\bx_1) \rmd \bx_1 + \sum_{i \ge 1} \ba_i \nu_i$. 
\end{definition}
More general conditions can be established, but we adopt the above for simplicity. We assume that the distribution over initial state $\bx_1 \sim \rho$ satisfies decomposability, which in particular encompasses the deterministic distribution over initial states considered in the body of the paper. The following lemma formalizes \Cref{lem:FOBG_bias}.  
\begin{proposition}\label{prop:lipschitz_rl} Suppose that $(\rho,p,\phi,c_{1:H};\pi_{1:H})$ is a benign Lipschitz planning problem. If $\rho$ is decomposable, then
\begin{itemize}
\item[(a)] For each $\btheta$, there exists a set Lebesgue-measurable set $\cZ \subset \R^{n + m H}$ such that $\Pr_{\bx_1 \sim \rho, \bw_{1:H} \sim p^H}[(\bx_1,\bw_{1:H})\in \cZ] = 1$ and  $\btheta \mapsto V_1(\bx_1,\bw_{1:H},\btheta)$ is differentiable for all $(\bx_1,\bw_{1:H})\in \cZ$. 
\item[(b)] $\nabthet F(\btheta) = \Exp_{\bx_1 \sim \rho, \bw_{1:H} \sim p^H}[ \nabla V_1(\bx_1,\bw_{1:H},\btheta)]$, where expectations are taken in the sense of \Cref{exp:ae}.
\end{itemize}
If $\rho$ is not necessarilly decomposable, but for given $\btheta \in \R^D$, the set $\{(\bx_1,\bw_{1:H}): V_1(\bx_1,\bw_{1:H},\btheta) \text{ is differentiable}\}$ is Lebesgue measureable, then points (a) and (b) still hold. 
\end{proposition}
\begin{example}[Piecewise linear] As an example, piecewise linear, or piecewise-polynomial dynamics statisfy the conditions of the above proposition. 
\end{example}
\subsubsection{Separable functions}

A key step in establishing the unbiasedness of the zeroth-order estimator for policy optimization is the special case of separable functions. We begin by stating guarantees for simple functions which the noise enters in the following  \emph{separable} fashion. 
\begin{definition}[Benign separability] We say that a function  $f(\btheta,\bw)$ has  \emph{benign separability} if there exists an everywhere differentiable function $\gin(\btheta)$ and a Lebesgue measurable function $\gout(\cdot)$ with polynomial growth such that
\begin{align*}
f(\btheta,\bw) = \gout(\gin(\btheta) + \bw).
\end{align*}
\end{definition}
A slightly more general version of the above definition is as follows. 
\begin{definition}
We say that a $\bx$-parameterized function function $f(\btheta,\bw; \bx)$ has \emph{parametrized benign separability} if there exists Lebesgue-measure functions $\gout(\cdot; \cdot)$ and $\gin(\cdot ; \cdot)$ such that $\gin(\cdot;\cdot)$ is differentiable for all $\bx$, and
\begin{align*}
f(\btheta,\bw ;\bx) = \gout(\gin(\btheta;\bx) + \bw;\bx),
\end{align*}
where (a) $(\bz,\bx) \mapsto \gout(\bz; \bx)$ has polynomial growth, (b) for each $\btheta$,   the mapping $\bx \mapsto \rmD_{\btheta} \gin(\btheta;\bx)$ has polynomial growth, $(\bz,\bx) \mapsto \gout(\bz; \bx)$ has polynomial growth, and  (c) for some $\epsilon_0 > 0$,  there is a function $\tilde{g}(\bx)$ with polynomial growth such that such that for  all $\Delta:\|\Delta\| \le \epsilon$,
\begin{align}
\| \gin(\btheta;\bx) - \gin(\btheta+\Delta;\bx) - \rmD \gin(\btheta;\bx) \cdot \Delta \|  \le \|\Delta\|^2 \tilde{g}(\bx) \label{eq:par_gin_condition}.
\end{align}
\end{definition}
We note that \Cref{eq:par_gin_condition} is satisfied when $\gin(\btheta;\bx)$ has a second derivative by having polynomial growth.

The following gives an expression for the derivative of separable functions. Note that we do not require $\gout(\cdot)$ to be differentiable, and depend only on the derivatives of $\psi(\bw) = \log p(\bw) + \text{ const.}$  from the density $p$, as well as the derivative of $\gin(\btheta)$. The following statement establishes the well-known \cite{reinforce} computation at our level of generality.
\begin{proposition}\label{prop:reinforce} Suppose that $p$ satisfies \Cref{asm:w_assumptions}. Then, if $f(\btheta,\bw)$ is benign separable, then the expectation $F(\btheta) = \Exp_{\bw \sim p}[f(\btheta,\bw)]$ is well defined, differentiable, and has
    \begin{align*}
    \nabla F(\btheta) = \Exp_{\bw \sim p}[\rmD \gin(\btheta)^\intercal \nabla \psi(\bw) \cdot  f(\btheta,\bw) ]. 
    \end{align*}
More generally, if $\rho$ has finite moments and $f(\btheta,\bw;\bx)$ has benign parametrized separability, then $F(\btheta) = \Exp_{\bx \sim \rho,\bw \sim p}[f(\btheta,\bw;\bx)]$ satisfies
\begin{align*}
    \nabla F(\btheta) = \Exp_{\bx \sim \rho,\bw \sim p}[\rmD \gin(\btheta)^\intercal \nabla \psi(\bw) \cdot f(\btheta,\bw;\bx) ]. 
    \end{align*}
\end{proposition}

\subsection{Proofs}\label{app:proofs_for_unbiasedness}
\subsubsection{Proof of \Cref{prop:reinforce}}
    \begin{lemma}\label{lem:w_facts} Let $p$ be the distribution of $\bw$ satisfying \Cref{asm:w_assumptions} (and, by abuse of notation, its density with respect to the Lebesgue measure). Then, the following statements are true
    \begin{itemize}
        \item[(a)] The distribution $p$ has finite moments. In particular, for any function $g(\cdot):\R^m \to \R^d$ with polynomial growth, $\Exp_{\bw\sim p }[\|g(\bw)\|] < \infty$. This only requires \Cref{asm:w_assumptions} part (a).
       \item[(b)] $\nabla \psi(\bw)$ has polynomial growth, and $\Exp[\nabla \psi(\bw)] = 0$. 
        \item[(c)] For any $B > 0$, there exists a function with polynomial growth such that $\tilde{g}(\cdot)$, for all $\Delta:\|\Delta\| \le B$,
        \begin{align*}
        |p(\bw) - p(\bw+\Delta) - p(\bw) \langle -\nabla \psi(\bw),\Delta \rangle| \le \|\Delta\|^2 p(\bw)\cdot \tilde{g}(\bw). 
        \end{align*}
        \item[(d)] Let $g(\cdot,\cdot): \R^m \times \R^n \to \R$ have polynomial growth. Then $\bx \mapsto \Exp_{\bw}  \nabla \psi(\bw)\cdot g(\bw,\bx)$ is well defined and polynomial growth in $\bx$. 
    \end{itemize}
    \end{lemma}
    \begin{proof} Since $p(\bw)$ decays exponentially in $\bw$, $p$ has finite moments. Thus, part (a) follows from \Cref{lem:poly_growth}. Part (b) follows from the fundamental theorem of calculus:
    \begin{align*}
    \|\nabla \psi(\bw)\| = \left\|\int_{0}^1  \nablatwo \psi(t\bw) \cdot \bw \rmd t\right\| \le  \|\nabla \psi(0)\| + \|\bw\|\max_{t \in [0,1]}\|\nablatwo \psi(t\bw)\|_{\op},
    \end{align*}
    the upper bound on which has polynomial growth since $\|\nablatwo \psi(t\bw)\|_{\op}$ does. 

    To prove part (c), we have that since $p(\bw) = e^{\alpha - \psi(\bw)}$ for $\psi$ differentiable 
    \begin{align*}
    \nabla p(\bw) =  - p(\bw) \cdot \nabla \psi(\bw), \quad \nablatwo p(\bw) = p(\bw) \cdot \underbrace{\left(\nabla \psi(\bw)\nabla \psi(\bw)^\intercal - \nablatwo \psi (\bw)\right)}_{:=M(\bw)}.
    \end{align*}
    Note that, by part (b) and the map that $\nablatwo \psi (\bw)$ has polynomial growth, $M(\bw)$ has polynomial growth. Therefore, for any bound $B > 0$, the function $\tilde{g}(\bw) := \sup_{\|\Delta\| \le B}M(\bw+\Delta)$ has polynomial growth. Finally, by the intermediate value theorem and for any $\Delta: \|\Delta\| \le B$,
    \begin{align*}
     |p(\bw) - p(\bw+\Delta) - p(\bw) \langle \nabla \psi(\bw), \Delta\rangle |   
     &= |p(\bw) - p(\bw+\Delta) - \langle \nabla p(\bw), \Delta\rangle |  \\
     &\le \|\Delta\|^2 p(\bw) M(\bw+t\Delta), \quad \text{ for some } t \in [0,1]\\
     &\le \|\Delta\|^2 p(\bw) \tilde{g}(\bw), 
    \end{align*}
    as needed.

    Part (d) is a consequence of part (b) and \Cref{lem:poly_growth_marginalization}.
    \end{proof}

    \begin{proof}[Proof of \Cref{prop:reinforce}] 

    Consider the non-parametric case. Since $\gout(\cdot)$ has polynomial growth, one can verify that $\bw \mapsto f(\btheta,\bw)$ has polynomial growth. Hence the expectation $F(\btheta)$ is well-defined by \Cref{lem:w_facts}, part (a). We now prove that $F(\btheta)$ is differentiable.  Fix a $\btheta$, and let $\|\btheta' - \btheta\| \le \epsilon$. 
    \begin{align*}
    F(\btheta) - F(\btheta') &= \int \left(f(\btheta,\bw) - f(\btheta',\bw)\right) p(\bw) \rmd \bw\\
    &= \int \left(\gout(\gin(\btheta) + \bw) - \gout(\gin(\btheta')+\bw) \right) p(\bw) \rmd \bw\\
    &= \int \gout(\underbrace{\gin(\btheta) + \bw}_{\bw_1})p(\bw) \rmd \bw - \int \gout(\underbrace{\gin(\btheta')+\bw)}_{\bw_2}  p(\bw) \rmd \bw\\
    &= \int \gout(\bw_1)p(\bw_1 - \gin(\btheta))\rmd \bw_1 -  \int \gout(\bw_2)p(\bw_2 - \gin(\btheta'))\rmd \bw_2\\
     &= \int \left(p(\bw - \gin(\btheta)) - p(\bw - \gin(\btheta'))\right)\cdot \gout(\bw)\rmd \bw\\
     &= \int \left(p(\bw ) - p(\bw + \gin(\btheta)- \gin(\btheta'))\right)\cdot \gout(\bw+\gin(\btheta))\rmd \bw\\
     &= \int \left(p(\bw ) - p(\bw + \gin(\btheta)- \gin(\btheta'))\right)\cdot f(\btheta,\bw)\rmd \bw\\
     &= \Exp_{\bw\sim p} \left[\left(\frac{p(\bw) - p(\bw + \gin(\btheta)- \gin(\btheta'))}{p(\bw)}\right)\cdot f(\btheta,\bw)\right].
    \end{align*}
    Setting $\Delta = \gin(\btheta)- \gin(\btheta')$, \Cref{lem:w_facts} implies that the remainder term enjoyes the following property.
    \begin{align*}
    R(\bw) := p(\bw) - p(\bw+\Delta) - p(\bw) \langle - \nabla \psi(\bw), \Delta\rangle \text{ satisifies } |R(\bw)| \le \|\Delta\|^2 \tilde{g}(\bw) p(\bw),
    \end{align*}
    where $\tilde{g}(\bw)$ has polynomial growth, and thus  $\tilde{g}(\bw)\cdot f(\btheta,\bw)$  integrable under $p$.  Thus, there exists a constant $C_w > 0$ such that
    \begin{align*}
    \left|\left(F(\btheta) - F(\btheta')\right) - \Exp_{\bw\sim p} \left[\langle - \nabla \psi(\bw), \Delta\rangle  f(\btheta,\bw)\right]\right| \le C_w \|\Delta\|^2,
    \end{align*}
    where the integral on the right hand side exists because $\psi(\bw)$ and $f(\btheta,\bw)$ have polynomial growth. Simplying and dividing by $\|\btheta - \btheta'\|$ and substituting again$\Delta = \gin(\btheta)- \gin(\btheta')$,
    \begin{align}\label{eq:F_diff_eq}
    \left|\frac{F(\btheta) - F(\btheta') - \left\langle \gin(\btheta') - \gin(\btheta), \Exp_{\bw} \left[ \nabla \psi(\bw)\cdot f(\btheta,\bw)\right] \right\rangle }{\|\btheta - \btheta'\|}\right|  \le C_w \frac{\| \gin(\btheta') - \gin(\btheta)\|^2}{\|\btheta - \btheta'\|}.
    \end{align}
    The result now follows from taking $\|\btheta - \btheta'\| \to 0$ and using differentiability of $\gout(\cdot)$ concludes.

    \paragraph{Parametrized case.} Now consider the parametrized case, and define $F(\btheta ; \bx) = \Exp_{\bw \sim p}f(\btheta,\bw;\bx)$. Then the analogue of \Cref{eq:F_diff_eq} holds pointwise for each $\bx$:
    \begin{align}\label{eq:F_diff_eq}
    \left|\frac{F(\btheta;\bx) - F(\btheta';\bx) - \left\langle \gin(\btheta';\bx) - \gin(\btheta;\bx), \Exp_{\bw} \left[ \nabla \psi(\bw;\bx)\cdot f(\btheta,\bw;\bx)\right] \right\rangle }{\|\btheta - \btheta'\|}\right|  \le C_w \frac{\| \gin(\btheta';\bx) - \gin(\btheta;\bx)\|^2}{\|\btheta - \btheta'\|}.
    \end{align}
    Using \Cref{eq:par_gin_condition}, the triangle inequality and Cauchy Schwartz, we obtain, for some integrable function $\tilde{g}$ with polynomial growth,
    \begin{align*}
    &\left|\frac{F(\btheta;\bx) - F(\btheta';\bx) - \left\langle\rmD \gin(\btheta;\bx)(\btheta'-\btheta), \Exp_{\bw} \left[ \nabla \psi(\bw;\bx)\cdot f(\btheta,\bw;\bx)\right] \right\rangle }{\|\btheta - \btheta'\|}\right|  \\
    &\qquad\le  \tilde{g}(\bx) \cdot \|\btheta-\btheta'\| \cdot\Exp_{\bw} \left\| \nabla \psi(\bw)\cdot f(\btheta,\bw;\bx)\right\|  +    C_w \frac{\| \gout(\btheta';\bx) - \gout(\btheta;\bx)\|^2}{\|\btheta - \btheta'\|}.
    \end{align*}
    Applying \Cref{eq:par_gin_condition} again, we can bound
    {\small
    \begin{align*}
    \| \gout(\btheta';\bx) - \gout(\btheta;\bx)\|^2 &= \left\| \left(\gout(\btheta';\bx) - \gout(\btheta;\bx) - \|\btheta-\btheta'\|\nabla_{\btheta}\gout(\btheta;\bx)\right) + \|\btheta-\btheta'\|\nabla_{\btheta}\gout(\btheta;\bx)\right\|^2 \\
    &= 2\|\btheta-\btheta'\|^2\|\nabla_{\btheta}\gout(\btheta;\bx)\|^2  + 2\left\| \gout(\btheta';\bx) - \gout(\btheta;\bx) - \|\btheta-\btheta'\|\nabla_{\btheta}\gout(\btheta;\bx) \right\|^2\\
    &\le 2\|\btheta-\btheta'\|^2\|\nabla_{\btheta}\gout(\btheta;\bx)\|^2  + 2 \tilde{g}(\bx)^2\|\btheta-\btheta'\|^4.
    \end{align*}}
    Thus,
    {\small
    \begin{align*}
    &\left|\frac{F(\btheta;\bx) - F(\btheta';\bx) - \left\langle \rmD \gin(\btheta;\bx)(\btheta'-\btheta), \Exp_{\bw} \left[ \nabla \psi(\bw;\bx)\cdot f(\btheta,\bw;\bx)\right] \right\rangle }{\|\btheta - \btheta'\|}\right|   \\
    &\qquad\le  \tilde{g}(\bx) \cdot \|\btheta-\btheta'\| \cdot\left\|\Exp_{\bw\sim p} )\nabla \psi(\bw)\cdot f(\btheta,\bw;\bx)\right\|  +    2C_w\|\btheta-\btheta'\|\|\nabla_{\btheta}\gout(\btheta;\bx)\|^2  + 2 C_w\tilde{g}(\bx)^2\|\btheta-\btheta'\|^3.
    \end{align*}}
    To conclude, observe that, by assumption, $\tilde{g}(\bx)$, $\bx \mapsto \nabla_{\btheta} \gout(\btheta;\bx)$, and (by  \Cref{lem:w_facts} part (d)) $\bx \mapsto \Exp_{\bw\sim p} \left[ \nabla_{\bw} \psi(\bw)\cdot f(\btheta,\bw;\bx)\right]$ all have polynomial growth. Thus, the fact that $\rho$ has finite moments ensures that all terms have expectations under $\bx \sim \rho$, and thus, 
    {\small
    \begin{align*}
    &\left|\frac{\Exp_{\bx \sim \rho}[F(\btheta;\bx) - F(\btheta';\bx)] - \Exp_{\bx \sim \rho}[\left\langle\rmD \gin(\btheta;\bx)(\btheta'-\btheta), 
    \Exp_{\bw} \left[ \nabla \psi(\bw;\bx)\cdot f(\btheta,\bw;\bx)\right] \right\rangle ]}{\|\btheta - \btheta'\|}\right| \\
    &\quad\le  \Exp_{\bx \sim \rho}\left|\frac{F(\btheta;\bx) - F(\btheta';\bx) - \left\langle \rmD \gin(\btheta;\bx)(\btheta'-\btheta), \Exp_{\bw} \left[ \nabla \psi(\bw;\bx)\cdot f(\btheta,\bw;\bx)\right] \right\rangle }{\|\btheta - \btheta'\|}\right|\\
    &\quad\le  \Exp_{\bx \sim \rho} \left[\tilde{g}(\bx) \cdot \|\btheta-\btheta'\| \cdot\left\|\Exp_{\bw\sim p}  \nabla \psi(\bw)\cdot f(\btheta,\bw;\bx)\right\|  +    2C_w\|\btheta-\btheta'\|\|\nabla_{\btheta}\gout(\btheta;\bx)\|^2  + 2 C_w\tilde{g}(\bx)^2\|\btheta-\btheta'\|^3\right]. 
    \end{align*}}
    Again, since all expectations are finite, the all terms on the last line above tend to zero as $\|\btheta-\btheta'\| \to 0$, so that 
    \begin{align*}
    \nabla_{\btheta} F(\btheta) &= \nabla_{\btheta} \left(\Exp_{\bx \sim \rho}[F(\btheta;\bx)]\right)\\
    &= \Exp_{\bx \sim \rho}\left[ \rmD \gin(\btheta;\bx)^\intercal \Exp_{\bw\sim p} \left[ \nabla_{\bw} \psi(\bw)\cdot f(\btheta,\bw;\bx)\right]\right]\\
    &= \Exp_{\bx \sim \rho,\bw \sim p}\left[ \nabla_{\btheta} \rmD \gin(\btheta;\bx)^\intercal   \nabla_{\bw} \psi(\bw)\cdot f(\btheta,\bw;\bx)\right]
    \end{align*}
    where in the last step, measurability and polynomial-growth conditions allow the application of Fubini's theorem. 
    \end{proof}

\subsubsection{Proof of \Cref{prop:reinforce_rl}}
    
    We prove a slightly different proof from that of the standard REINFORCE lemma to accommodate the fact that the state space is continuous, but the distribution over states may not have a density with respect to the Lebesgue measure. Instead, we adopt an approach based on the performance difference lemma \citep[Lemma 5.2.1]{kakade2003sample}.

    To begin, define the expected cost to go function and expected costs
    \begin{align}
    \bar{V}_h (\bx_h,\btheta) &= \Exp_{\bw_{h:H}\iidsim p} V(\bx_h,\bw_{h:H},\btheta) \tag{a}\\
    {V}_{h,\btheta}^+ (\btheta',\bw_h ; \bx_h) &= \bar{V}_{h+1}( \phi(\bx_h,\pi_h(\bx_h,\btheta') + \bw_t),\btheta)\tag{b.1}\\
    \bar{V}_{h,\btheta}^+ (\btheta' ; \bx_h ) &= \Exp_{\bw_{h}\sim p} {V}_h^+ (\btheta',\bw_h; \bx_h,\btheta ). \tag{b.2}\\
    c_h (\btheta,\bw_h ; \bx_h) &= c_h(\bx_h,\pi_h(\bx_h,\btheta) + \bw_t),\quad  \tag{c.1}\\
    \bar{c}_h (\btheta ; \bx_h) &= \Exp_{\bw_{h}\sim p} c_h (\btheta,\bw_h ; \bx_h) \tag{c.2}
    \end{align}
    which describe (a) the expected cost-to-go under $\bx_h,\btheta$, and (b) the expected cost-to-go from the next stage $h$ after starting in state $\bx_h$, acting according to $\btheta'$ in stage $h$, and subsequently acting according to $\btheta$, and (c)  expected cost in state $\bx_h$ under policy $\theta$ and. By the well known performance-difference lemma, we have
    \begin{align}
    &F(\btheta) - F(\btheta') \\
     &= \Exp_{\bx_1 \sim \rho} \left[\bar{V}_1 (\bx_1,\btheta) - \bar{V}_1 (\bx_1,\btheta')\right]\nonumber\\
    &= \sum_{h=1}^H \Exp_{\btheta;h}[\left(\bar{c}_h (\btheta ; \bx_h) - \bar{c}_h (\btheta' ; \bx_h)\right) + \left(\bar{V}_{h,\btheta}^+ (\btheta ; \bx_h) - \bar{V}_{h,\btheta}^+ (\btheta' ;\bx_h  )\right)]\nonumber\\
    &= \sum_{h=1}^H \Exp_{\btheta;h}\Exp_{\bw_h \sim p}[c_h (\btheta,\bw_h ; \bx_h) - c_h (\btheta',\bw_h h \bx_h)] + \Exp_{\bx_h \sim \btheta;h}\Exp_{\bw_h \sim p}[V_{h,\btheta}^+ (\btheta,\bw_h ; \bx_h) - V_{h,\btheta}^+ (\btheta',\bw_h ; \bx_h)]\nonumber\\
    &= \sum_{h=1}^H (F_{h,\btheta;c}(\btheta)-F_{h,\btheta;c}(\btheta')) + (F_{h,\btheta;V}(\btheta) - F_{h,\btheta;c}(\btheta')) \label{eq:FcFV}
    \end{align}
    where $\Exp_{\btheta,h}$ denotes  expectations over $\bx_h$ under the dynamics
    \begin{align*}
    \bx_1 \sim \rho, \quad \bx_{t+1} = \phi(\bx_t,\bu_t), \quad \bu_t = \pi(\bx_t,\btheta) + \bw_t,
    \end{align*}
    and where we define
    \begin{align*}
    F_{h,\btheta;c}(\btheta') := \Exp_{\bx_h \sim \btheta;h}\Exp_{\bw_h \sim p}c_h (\btheta',\bw_h; \bx_h)],\quad  F_{h,\btheta;V}(\btheta') := \Exp_{\bx_h \sim \btheta;h}\Exp_{\bw_h \sim p}V_{h,\btheta}^+ (\btheta',\bw_h \mid \bx_h)].
    \end{align*}
    Hence, if the functions $F_{h,\btheta;c}(\btheta')$ and $F_{h,\btheta;c}(\btheta')$ are differentiable at $\btheta' = \btheta$ for $h = 1,2,\dots,H$, \Cref{eq:FcFV} implies
    \begin{align}
    \nabla_{\btheta} F(\btheta) = \sum_{h=1}^H \left(\nabla_{\btheta'} F_{h,\btheta;c}(\btheta') + \nabla_{\btheta'} F_{h,\btheta;V}(\btheta')\right) \big{|}_{\btheta' = \btheta}.  \label{eq:nab_b_F_V_c}
    \end{align}
    We establish differentiability and compute the derivatives by appealing to \Cref{prop:reinforce}. First, we establish a couple of useful claims. 
    \begin{claim}\label{claim:dist_all_moments} The marginal distribution over $\bx_h$ under $\Exp_{h;\btheta}$ has all moments.
    \end{claim}
    \begin{proof} Observe that the polynomial growth conditions on the dynamics map $\phi(\cdot,\cdot)$ imply that as a function, $\bx_h = \bx_h(\bx_1,\bw_{1:h-1})$, $\bx_h$ has polynomial growth in $\bx_h(\bx_1,\bw_{1:h-1})$. Thus, since the distributions over $\bx_1$ and $\bw_{1:h-1}$ have all moments, so does the distribution over $\bx_h$. 
    \end{proof}
    \begin{claim}\label{claim:c_claim} The function $\btheta \mapsto c_h (\btheta,\bw_h ; \bx_h) = c_h(\bx_h,\pi(\bx_h;\btheta) + \bw_t)$ satisfies benign parametrized separability. 
    \end{claim}
    \begin{proof} Take $\gout(\cdot;\bx_h) = c_h(\bx_h,\cdot)$ and $\gin = \pi(\bx_h,\btheta)$. Since $c_h(\cdot,\cdot)$ has polynomial growth, the requisite growth condition on $\gout(\cdot,\cdot)$ holds. The polynomial growth in $\bx$ of the second-order differentials of $\btheta \mapsto \pi_h(\bx,\btheta)$ implies that the first order differential of $\btheta \mapsto \pi_h(\bx,\btheta)$ has polynomial growth in $\bx$, and that $\gin$ also satisfies \Cref{eq:par_gin_condition} by Taylor's theorem. Hence, $\gout,\gin$ satisfy the requisite conditions. 
    \end{proof}

    \begin{claim} The function $\btheta \mapsto V_{h,\btheta_0}^+ (\btheta,\bw_h ; \bx_h) = \bar{V}_{h+1}( \phi(\bx_h,\pi_h(\bx_h,\btheta') + \bw_t),\btheta_0)$ satisfies benign parametrized separability. 
    \end{claim}
    \begin{proof} Take $\gout(\bu;\bx_h) = \bar{V}_{h+1}( \phi(\bx_h,\bu),\btheta_0)$ and $\gin = \pi_h(\bx_h,\btheta)$. As shown in \Cref{claim:c_claim}, $\gin$ satisifes the requisite conditions for  benign parametrized separability. To conclude, it suffices to show that $(\bu,\bx_h) \mapsto \bar{V}_{h+1}( \phi(\bx_h,\bu),\btheta_0)$ has polynomial growth. By \Cref{lem:poly_growth_marginalization}, it suffices to show that
    \begin{align*}
    (\bu,\bx_h,\bw_{h+1:H}) \mapsto V_{h+1}( \phi(\bx_h,\bu),\bw_{h+1:H}\btheta_0) 
    \end{align*}
    has polynomial growth. This holds since we have
    \begin{align*}
    V_{h+1}( \phi(\bx_h,\bu),\bw_{h+1:H},\btheta_0)  &= \sum_{i = h+1}^H c_i(\bx_i,\pi_h(\bx_i,\btheta_0)+\bw_i, \quad \text{s.t. } \bx_{i+1} = \phi(\bx_i,\pi_h(\bx_i,\btheta_0))+\bw_i.
    \end{align*}
    Just as in the proof of \Cref{claim:dist_all_moments}, $\bx_i, i > 1$ have polynomial growth when viewed as functions of $\bw_{h+1:H}$,  $\bx_h$ and $\bu$ (since the dynamics $\phi$) have polynomial growth. Since $c_i$ also have polynomial growth, we conclude $V_{h+1}( \phi(\bx_h,\bu),\bw_{h+1:H},\btheta_0) $ must as well. 
    \end{proof}
    The above three claims allow us to invoke \Cref{prop:reinforce}, so that
    \begin{align*}
    \nabla_{\btheta'} F_{h,\btheta;c}(\btheta') &= \Exp_{\bx_h \sim \btheta;h}\Exp_{\bw_h \sim p} \left[\rmD_{\btheta'} \pi_h(\bx_h,\btheta')^\intercal \nabla \psi(\bw_h) c_h (\btheta',\bw_h ;\bx_h)]\right]\\
    \nabla_{\btheta'} F_{h,\btheta;V}(\btheta') &= \Exp_{\bx_h \sim \btheta;h}\Exp_{\bw_h \sim p}\left[\rmD_{\btheta'} \pi_h(\bx_h,\btheta')^\intercal \nabla \psi(\bw_h) \cdot V_{h,\btheta}^+ (\btheta',\bw_h \mid \bx_h)\right].
    \end{align*}
    Therefore, from \Cref{eq:nab_b_F_V_c}, we conclude
    \begin{align*}
    \nabthet F(\btheta)&= \sum_{h=1}^H  \Exp_{\bx_h \sim \btheta;h}\Exp_{\bw_h \sim p}\left[\rmD_{\btheta'} \pi_h(\bx_h,\btheta)^\intercal \nabla \psi(\bw_h) \left( c_h (\btheta',\bw_h ;\bx_h) +  V_{h,\btheta}^+ (\btheta,\bw_h \mid \bx_h)\right)]\right].
    \end{align*}
    Thus, the various polynomial growth conditions imply we can use Fubini's theorem (and the definition of $V_{h,\btheta}^+$) , so that the above is equal to
    \begin{align*}
    \nabthet F(\btheta)&= \sum_{h=1}^H  \Exp_{\bx_1\sim \rho}\Exp_{\bw_{1:H} \sim p^H}\left[\rmD_{\btheta} \pi_h(\bx_h,\btheta)^\intercal \nabla \psi(\bw_h) \left( c_h (\bx_h,\bu_h) +  V_{h+1} (\bx_h,\bw_{h+1:H},\btheta)\right)]\right]\\
    &= \sum_{h=1}^H  \Exp_{\bx_1\sim \rho}\Exp_{\bw_{1:H} \sim p^H}\left[\rmD_{\btheta} \pi_h(\bx_h,\btheta)^\intercal \nabla \psi(\bw_h) \cdot V_{h} (\bx_h,\bw_{h:H},\btheta)]\right]\\
    &= \Exp_{\bx_1\sim \rho}\Exp_{\bw_{1:H} \sim p^H}\left[\sum_{h=1}^H  \rmD_{\btheta} \pi_h(\bx_h,\btheta)^\intercal \nabla \psi(\bw_h) \cdot V_{h} (\bx_h,\bw_{h:H},\btheta)\right].
\end{align*}
This completes the first part of the proof. Next, we simplify in the special case where $\Exp_{\bw \sim p}[\nabla \psi(\bw)] = 0$. Observe that the last line of the above display is equal to 
\begin{align*}
    & \Exp_{\bx_1\sim \rho}\Exp_{\bw_{1:H} \sim p^H}\left[\sum_{h=1}^H  \rmD_{\btheta} \pi_h(\bx_h,\btheta)^\intercal \nabla \psi(\bw_h) \cdot V_{1} (\bx_1,\bw_{1:H},\btheta)\right]\\
    &\qquad- \underbrace{\Exp_{\bx_1\sim \rho}\Exp_{\bw_{1:H} \sim p^H}\left[\sum_{h=1}^H  \rmD_{\btheta} \pi_h(\bx_h,\btheta)^\intercal \nabla \psi(\bw_h) \cdot \left(\sum_{i=1}^{h-1} c_i(\bx_i,\bu_i)\right) \right]}_{(b)},
    \end{align*}
    where in the last line, we use that $V_{1} (\bx_1,\bw_{1:H},\btheta) = \sum_{i=1}^{h-1} c_i(\bx_i,\bu_i) = V_{h} (\bx_h,\bw_{1:H},\btheta)$. It suffices to show term $(b)$ is zero. This follows since, for each $i < h$, we have
    \begin{align*}
     &\Exp_{\bx_1\sim \rho}\Exp_{\bw_{1:H} \sim p^H}\left[\rmD_{\btheta} \pi_h(\bx_h,\btheta)^\intercal \nabla \psi(\bw_h) c_i(\bx_i,\bu_i)\right] \\
     &\qquad= \Exp_{\bx_1\sim \rho}\Exp_{\bw_{1:h-1} \sim p^{h-1}}[\rmD_{\btheta} \pi_h(\bx_h,\btheta)^\intercal c_i(\bx_i,\bu_i)] \cdot  \Exp_{\bw_h\sim p}[\nabla \psi(\bw_h)] = 0.
    \end{align*}
    Here, we used that $\bw_h$ is independent of $\bx_1,\bw_{1:h-1}$, and the assumption that $\Exp_{\bw_h\sim p}[\nabla \psi(\bw_h)] = 0$.

\subsubsection{Proof of \Cref{prop:lipschitz_rl}}
\newcommand{\phitil}{\tilde{\phi}}
\newcommand{\ctil}{\tilde{c}}
\newcommand{\deltheta}{\boldsymbol{\delta}}


First, we establish almost-everywhere differentiability. 
Let $\phitil_h$ denote the transitions under noise $\bw$ and policy $\btheta$, defined as
\begin{align*}
\phitil_h(\bx,\bw,\btheta) &= \phi(\bx,\pi_h(\bx,\btheta)+\bw).
\end{align*}
Set $\Phi_h$ to be their composition
\begin{align*}
\Phi_h(\bx_1,\bw_{1:h-1},\btheta) = \phitil_{h-1}(\cdot,\bw_h,\btheta) \circ \phitil_{h-2}(\cdot,\bw_{h-1},\btheta) \circ \dots \circ \phitil_1(\bx_1,\bw_1,\btheta). 
\end{align*}
Notice that $\bx_h = \Phi_h(\bx_1,\bw_{1:h-1},\btheta)$, where $\bx_h$ is generated according to the dynamics  $\bx_{i+1} =  \phi(\bx_i, \pi_i(\bx_i,\btheta)+\bw_i)$. We now establish two key claims.
\begin{claim}\label{claim:value_Phi_diff_claim} Fix $(\bx_1,\bw_{1:H},\btheta)$. If  $\btheta' \mapsto \Phi_h(\bx_1,\bw_{1:h-1},\btheta')$ for all  is differentiable at $\btheta' = \btheta$ for all $h \in [H]$, then $\btheta' \mapsto V_1(\bx_1,\bw_{1:H},\btheta')$ is differentiable at $\btheta' = \btheta$. 
\end{claim}
\begin{proof} Defining $\ctil_{h}(\bx,\btheta;\bw) = c_h(\bx,\pi_h(\bx,\btheta)+\bw)$, we have
\begin{align*}
V_1(\bx_1,\bw_{1:H},\btheta) = \sum_{h=1}^H \ctil_{h}(\Phi_h(\bx_1,\bw_{1:h-1},\btheta),\btheta;\bw_h). 
\end{align*}
Since both $c_h(\cdot,\cdot)$ and $\pi_h(\cdot,\cdot)$ are everywhere differentiable (jointly in their arguments),  $(\bx,\btheta) \mapsto \ctil_{h}(\bx,\btheta;\bw) $ is everywhere differentiable. Thus, under the assumptions of the claim, the composition $\btheta' \mapsto \ctil_{h}(\Phi_h(\bx_1,\bw_{1:h-1},\btheta'),\btheta';\bw_h)$ is differentiable at $\btheta' = \btheta$. 
\end{proof}
The next claim provides a sufficient condition for \Cref{claim:value_Phi_diff_claim} to hold. 
\begin{claim}\label{claim:Phi_diff}  Fix $(\bx_1,\bw_{1:H},\btheta)$. Then if $\bw_{1:h-1}' \mapsto \Phi_h(\bx_1,\bw_{1:h-1},\btheta)$ is differentiable at $\bw_{1:h-1}' = \bw_{1:h-1}$ for all $h \in [H]$,  then $\btheta' \mapsto \Phi_h(\bx_1,\bw_{1:h-1},\btheta')$ is differentiable at $\btheta' = \btheta$ for all $h \in [H]$. 
\end{claim}
\begin{proof} Fix $\bx_1,\bw_{1:h-1}$. Let $\deltheta \in \R^d$ denote perturbations of $\btheta$. It suffices to show that, for each $h =1,2,\dots,H$, the mapping $\Psi_h(\deltheta) := \Phi_h(\bx_1,\bw_{1:h-1},\btheta+\deltheta)$  is differenitable if $\deltheta$.  By induction, it is straightforward to verify the identity
\begin{align}
\Psi_h(\deltheta) &:= \Phi_h(\bx_1,\bw_{1:h-1},\btheta+\deltheta) = \Phi_h(\bx_1,\bw_{1:h-1}+\tilde{\bw}_{1:h}(\deltheta)),\btheta)  \label{eq:Psi_h}
\end{align}
where we have defined the noise term $\tilde{\bw}_{1:h-1}(\deltheta)$  so as to transition from policy $\btheta$ to policy $\btheta + \deltheta$:
\begin{align}
\tilde{\bw}_{i}(\deltheta) = \pi_i(\Psi_{i}(\deltheta),\btheta+\deltheta) - \pi_i(\Psi_{i}(\deltheta),\btheta). \label{eq:tilw}
\end{align}
We now argue by induction on little $h$ that if $\bw_{1:i-1}' \mapsto \Phi_h(\bx_1,\bw_{1:i-1},\btheta)$ for all $i \le H$, then $\Psi_i(\deltheta)$ is differentiable at $\deltheta = 0$ for all $i \le h$.

For $h = 1$, both maps are the constant map $\Phi_1(\cdot,\cdot,\bx_1) = \bx_1$, so the result holds trivially. Now suppose the inductive hypothesis holds at some $h \ge 1$. Then, since each $\pi_i(\cdot,\cdot)$ is everywhere differentiable in its arguments, and since $\Psi_i(\deltheta)$ is differentiable at $\deltheta = 0$ for all $i \le h$ by inductive hypothesis,  $\tilde{\bw}_{i}(\deltheta)$ defined in \Cref{eq:tilw} is differentiable at $\deltheta = 0$ for each $i \le h$. Hence, $\tilde{\bw}_{1:h}(\deltheta)$ is differentiable at $\deltheta = 0$. Now, by assumption $\bw_{1:h}' \mapsto \Phi_{h+1}(\bx_1,\bw_{1:h}',\btheta)$ is differentiable at $\bw_{1:h}' = \bw_{1:h}$. Therefore, at $\deltheta = 0$, $\Psi_{h}(\deltheta)$ is given by the composition of two maps which are differentiable, and hence is differentiable. 
\end{proof}
\newcommand{\cW}{\mathcal{W}}
Define the set $\cW_{h}(\bx_1)$ as the set of $\bw_{1:H} \in \R^{mH}$ such that the map $\bw_{1:H} \mapsto \Phi_h(\bx_1,\bw_{1:h-1},\btheta)$ is differentiable (for simplicity, we augmented the map to be a function of all noises $\bw_{1:H}$). Synthesizing \Cref{claim:Phi_diff,claim:value_Phi_diff_claim}, we see that if  $\bw_{1:H} \in \bar{\cW}(\bx_1) := \bigcap_{h=1}^H\cW_{h}(\bx_1)$, then $\btheta' \mapsto V_1(\bx_1,\bw_{1:H},\btheta')$ is differentiable at $\btheta' = \btheta$. 

Furthermore, define  $\cZ_{h}$ as the set of $(\bx_1,\bw_{1:H}) \in \R^{d + mH}$ such that  $(\bx_1,\bw_{1:H}) \mapsto \Phi_h(\bx_1,\bw_{1:h-1},\btheta)$ is differentiable.  Here, we've just added $\bx_1$ as a nuissance variable, so \Cref{claim:Phi_diff,claim:value_Phi_diff_claim} also imply that, on $\bar{\cZ} := \bigcap_{h=1}^H\cZ_h$,  $\btheta' \mapsto V_1(\bx_1,\bw_{1:H},\btheta')$ is differentiable at $\btheta' = \btheta$.

We invoke Rademacher's theorem. Since $\bw_{1:H} \mapsto \Phi_h(\bx_1,\bw_{1:h-1},\btheta)$ is given by the composition of locally Lipschitz maps (note that differentiable maps are locally Lipschitz), \Cref{lem:Rademacher} implies that $\R^{mH} \setminus \cW_{h}(\bx_1)$ has Lebesgue measure zero for each $h$, so that $\R^{mH} \setminus \bar{\cW}(\bx_1)$ has Lebesgue measure zero by a union bound for each fixed $\bx_1$. Similarly, $\R^{d+mH} \setminus \bar{\cZ}$ has measure zero. 

\paragraph{Proof under decomposability.} Assume $\rho$ is decomposable with atoms $\ba_1,\ba_2,\dots$. Define the set
 \begin{align*}
 \cZ := \bar{\cZ} \cap \bigcap_{i\ge 1} \{\ba_i\} \times \bar{\cW}(\ba_i).
 \end{align*}
The set $\cZ$ is Lebesgue measurable because it is the intersection of Lebesgue measurable sets. Moroever, by the above discussion, $\btheta \mapsto V_1(\bx_1,\bw_{1:H},\btheta)$ is differentiable everywhere on $\cZ$. Lastly, one can verify by decomposability and the fact that $\bar{\cZ}$ and $\cW$ are the complement of Lebesgue measure-zero sets that $\Pr_{\bx_1 \sim \rho, \bw_{1:H} \sim \rho^H} [(\bx_1,\bw_{1:H}) \in \cZ] = 1$.  This proves part (a). 

To prove part (b), one can use the polynomial Lipschitz conditions to verify that $\nabla_{\btheta} V_1(\bx_1,\bw_{1:H},\btheta)$,   has polynomial growth wherever  defined. Hence, its expectation (in the sense of \Cref{exp:ae}) is well-defined. To prove part (b), one  can verify that, via polynomial-Lipschitzness of the dynamics, policies and costs that   the quotients satisfy
\begin{align*}
\frac{V_1(\bx_1,\bw_{1:H},\btheta) - V_1(\bx_1,\bw_{1:H},\btheta + \deltheta)}{\|\deltheta\|} \le \mathrm{poly}(\|\bx_1\|,\|\bw_{1:H}\|). 
\end{align*}
Hence, the quotients are uniformly integrable, and one can apply the dominate convergence theorem to show that, for any sequence $\deltheta_n \to 0$
\begin{align*}
\lim_{n \to \infty} \frac{F(\btheta+\deltheta_n)-F(\btheta)}{\|\deltheta_n\|} = \Exp_{\bx_1 \sim \rho, \bw_{1:H} \sim p^H}\left[\lim_{n \to \infty}\frac{V_1(\bx_1,\bw_{1:H},\btheta) - V_1(\bx_1,\bw_{1:H},\btheta + \deltheta_n)}{\|\deltheta_n\|}\right].
\end{align*}
By considering $\deltheta_n = t_n \bv$ for a direction $\bv \in \R^d$ and a sequence $t_n \to 0$, one can equate directional derivatives
\begin{align*}
\langle \nabthet F(\btheta), \bv \rangle = \Exp_{\bx_1 \sim \rho, \bw_{1:H} \sim p^H}[\langle  \nabthet V_1(\bx_1,\bw_{1:H},\btheta),\bv \rangle].
\end{align*}
This proves that\footnote{Note that , $\nabthet F(\btheta)$ is differentiable by \Cref{prop:reinforce_rl}, so we do not make the mistake of using existence of partial derivatives to imply differentiability.  }
\begin{align*}
\nabthet F(\btheta) = \Exp_{\bx_1 \sim \rho, \bw_{1:H} \sim p^H}[\nabthet V_1(\bx_1,\bw_{1:H},\btheta)].
\end{align*}

\paragraph{Proof under measurability assumption.} Consider the set
\begin{align*}
\cZ_0 := \{(\bx_1,\bw_{1:H}) \text{~~s.t.~~} \btheta' \mapsto V_1(\bx_1,\bw_{1:H},\btheta') \text{ is differentiable at }\btheta\}
\end{align*}
and define its slices
\begin{align*}
\cZ_0(\bx_1) := \{\bw_{1:H} \text{~~~s.t. } \btheta' \mapsto V_1(\bx_1,\bw_{1:H},\btheta') \text{ is differentiable at }  \btheta\}.
\end{align*}
If we assume that $\cZ_0 $ is Lebesgue measurable, then by Fubini's theorem,
\begin{align*}
\Pr_{\bx \sim \rho,\bw_{1:H}\sim p^H}[(\bx_1,\bw_{1:H}) \in \cZ_0] = \Exp_{\bx \sim \rho}\Pr_{\bw_{1:H}\sim p^H}[\bw_{1:H} \in \cZ_0(\bx_1)].
\end{align*}
Notice that, for any given $\bx_1$, the above proof under decomposability shows that $\cZ_0(\bx_1) \supseteq \bar{\cW}(\bx_1)$, and thus the complement of $\cZ_0(\bx_1)$  in $\R^{mH}$ has Lebesgue measure zero.  Hence $\Pr_{\bw_{1:H}\sim p^H}[\bw_{1:H} \in \cZ_0(\bx_1)] = 1$, so that $\Pr_{\bx \sim \rho,\bw_{1:H}\sim p^H}[(\bx_1,\bw_{1:H}) \in \cZ_0]  = 0$. This proves part (a). Part (b) follows by the same dominated convergence argument.\qed

%% file: appendix/max_ders_two.tex

\section{Additional Proofs from \Cref{sec:bias_variance}}

\subsection{Proof of \Cref{lem:variance_empirical_bias}}\label{sec:proof:lem_variance_empirical_bias}
Recall that empirical bias means there exists an event $\cE$ such that $\|\Exp[\bz \mid \cE] - \Exp[\bz]\| \ge \Delta$, and $\Pr[\cE] \ge 1-\beta$. Since the target lower bound increases as $\beta$ decreases, we may assume that $\Pr[\cE] = 1-\beta$ with equality (since choosing a small $\beta$ so that equality holds gives a larger variance lower bound). We begin
\begin{align*}
\Delta &\le \|\Exp[\bz \mid \cE] - \Exp[\bz]\| \\
&= \|(1-\beta)^{-1}\Exp[\bz \I\{\cE\} ] - \Exp[\bz]\|\\
&\le \|(1-\beta)^{-1}\Exp[\bz \I\{\cE\} ] + (1-\beta)^{-1}\Exp[\bz]\| - \|\Exp[\bz]\| \cdot |1-(1-\beta)^{-1}|\\
 &\le (1-\beta)^{-1}\|\Exp[  \bz \I\{\cE^c\} ]\|  - \|\Exp[\bz]\| \cdot |1-(1-\beta)^{-1}|.
 \end{align*}
 Rearranging, we have
 \begin{align*}
 \|\Exp[  \bz \I\{\cE^c\} ]\| \ge \Delta_0 := \max\{0,(1-\beta)\Delta - \beta \|\Exp[\bz]\|\}.
 \end{align*}
 And thus, since $\Pr[\cE^c] = \beta$,
 \begin{align*}
 \|\Exp[  \bz \mid \cE^c ]\| \ge \frac{\Delta_0}{\beta}. 
 \end{align*}
 Therefore, 
 \begin{align*}
 \Exp[  \|\bz\|^2  ] &\ge \Exp[  \|\bz\|^2 \I\{\cE^c\} ] \\
 &= \Pr[ \cE^c] \cdot \Exp[  \|\bz\|^2 \mid \cE^c ]\\
  &\ge \Pr[ \cE^c] \cdot \|\Exp[  \bz \mid \cE^c ]\|^2\\
&\ge\beta  \cdot \frac{\Delta_0^2}{\beta^2} = \frac{\Delta_0^2}{\beta} .
 \end{align*}
\qed 
\subsection{Proof of \Cref{lem:var_zero_order_one}}\label{sec:proof:lem_variance_zero_order}
Let's consider that $\nabhatzero$ estimator with a single sample, and drop the superscript $i$. We accommodate the general case with $\bx_1 \sim \rho$. Since $\Variance[\bz] \le \Exp[\|\bz\|^2]$ for any random vector $\bz$, we have
{\small
\begin{align*}
        \textbf{Var}\left[\frac{1}{\sigma^2}V_1(\bx_1, \bar{\bw}^i,\btheta) \left[\sum^H_{h=1} \nabla_{\btheta} \pi(\bx_h,\btheta)^\top \bw^i_h\right]\right] &\le \mathbb{E}_{\bar{\bw}^i,\bx_1}\left\|\frac{1}{\sigma^2}V_1(\bx_1, \bar{\bw}^i,\btheta) \cdot \sum^H_{h=1} \rmD_{\btheta} \pi(\bx_h^i,\btheta)^\top \bw^i_h\right\|^2_2 \\
        &\le \frac{B_V^2}{\sigma^4} \mathbb{E}_{\bar{\bw}^i,\bx_1}\left \| \sum^H_{h=1} \rmD_{\btheta} \pi(\bx_h,\btheta)^\top \bw_h^i\right\|^2_2 \\
        &= \frac{B_V^2}{\sigma^4} \mathbb{E}_{\bar{\bw}^i,\bx_1}\left[ \sum_{h_1=1}^H\sum^H_{h_1=1} \left\langle \rmD_{\btheta} \pi(\bx_{h_1},\btheta)^\top \bw_{h_1}^i,  \rmD_{\btheta} \pi(\bx_{h_2},\btheta)^\top \bw_{h_2}^i \right\rangle\right]\\
        &= \frac{B_V^2}{\sigma^4}  \sum_{h_1=1}^H\sum^H_{h_1=1}  \mathbb{E}_{\bar{\bw}^i,\bx_1}\left[\left\langle \rmD_{\btheta} \pi(\bx_{h_1},\btheta)^\top \bw_{h_1}^i,  \rmD_{\btheta} \pi(\bx_{h_2},\btheta)^\top \bw_{h_2}^i \right\rangle\right].
    \end{align*}
}

   We claim that $\mathbb{E}_{\bar{\bw}^i,\bx_1}\left[\left\langle \rmD_{\btheta} \pi(\bx_{h_1},\btheta)^\top \bw_{h_1}^i,  \rmD_{\btheta} \pi(\bx_{h_2},\btheta)^\top \bw_{h_2}^i \right\rangle\right] = 0$ unless $h_1 = h_2$. Suppose $h_1 \ne h_2$. Since inner products are symmetric, we may assume without loss of genearlity that $h_1 < h_2$. Then, $\bx_{h_2}$, $\bx_{h_1}$ and $\bw_{h_1}$are all functions of $\bx_{1}$ and $\bw_{1:h_2-1}$,whereas $\bw_{2}$ is independent of these. Hence, since $\Exp[\bw_2] = 0$, the cross term vanishes. Thus, we are left with
   {\small
   \begin{align*}
   \textbf{Var}\left[\frac{1}{\sigma^4}V_1(\bx_1, \bar{\bw}^i,\btheta) \left[\sum^H_{h=1} \nabla_{\btheta} \pi(\bx_h,\btheta)^\top \bw^i_h\right]\right]  &\le \frac{B_V^2}{\sigma^4}  \sum_{h=1}^H \mathbb{E}_{\bar{\bw}^i,\bx_1}\left[\left\langle \rmD_{\btheta} \pi(\bx_{h},\btheta)^\top \bw_{h}^i,  \rmD_{\btheta} \pi(\bx_{h},\btheta)^\top \bw_{h}^i \right\rangle\right]\\
   &\le \frac{B_V^2}{\sigma^4}  \sum_{h=1}^H \mathbb{E}_{\bar{\bw}^i,\bx_1}\left[ \|\rmD_{\btheta} \pi(\bx_{h},\btheta)\|_{\op} \|\bw_{h}^i\|^2 \right]\\
   &\le \frac{B_V^2  B_{\pi^2}}{\sigma^4}  \sum_{h=1}^H \mathbb{E}_{\bar{\bw}^i,\bx_1}\left[ \ \|\bw_{h}^i\|^2 \right] = \frac{ B_V^2 B_{\pi^2}}{\sigma^4} \cdot Hn \sigma^2 = \frac{Hn B_V^2 B_{\pi}^2}{\sigma^2},   
   \end{align*}}
   as needed.
\qed

%% file: appendix/derivations_interpolation.tex

\section{Interpolation}

\subsection{Bias and variance of the interpolated estimator}\label{app:interpolatedbiasvariance}

Here we describe the bias and variance of the interpolated estimator. The first is a straightforward consequence of linearity of expectation and the expectation computations in \Cref{eq:tractableform}.

\begin{lemma}[Interpolated bias]\label{lem:interpolated_bias} Assuming the costs and dynamics satisfies the conditions of \Cref{lem:ZOBG_bias} (formally, \Cref{cor:gaussian_zobg}), then for all $\alpha \in [0,1]$,
\begin{align*}
\Exp[\nabbaralpha F(\btheta)] - \nabla F(\btheta) = \alpha \left(\Exp[\nabbarone F(\btheta)] - \nabla F(\btheta)\right).
\end{align*}
If in addition, the costs and dynamics satisfy the conditions of \Cref{lem:FOBG_bias} (formally, \Cref{prop:lipschitz_rl}), then $\Exp[\nabbaralpha F(\btheta)] = \nabla F(\btheta)$.
\end{lemma}

\begin{lemma}[Interpolated variance]\label{lem:interpolated_variance} Assume that $\nabbarone F(\btheta)$ and $\nabbarzero F(\btheta)$ are constructed using two independent sets of $N$ trajectories. Then We have that
\begin{align*}
\Variance[\nabbaralpha F(\btheta)] &=  \alpha^2 \Variance[\nabbarone F(\btheta)] + (1-\alpha)^2 \Variance[\nabbarzero F(\btheta)] \\ 
&= \frac{\alpha^2}{N} \Variance[\nabhatzero F_i(\btheta)] + \frac{(1-\alpha)^2 }{N} \Variance[\nabhatone F_i(\btheta)].
\end{align*}
\end{lemma}
\begin{proof} Let $X = \nabbarone F(\btheta)$ and $Y = \nabbarone F(\btheta)$. Since the \zobg{} and \fobg{} are assumed to use independent trajectories, $X$ and $Y$ are independent, and thus 
\begin{align*}
\Variance[\alpha X+(1-\alpha) Y] &= \Exp[\|\alpha (X -\Exp[X]) + (1-\alpha) (Y-\Exp[Y])\|^2 ] \\
&= \alpha^2\Exp\|X-\Exp[X]\|^2 + (1-\alpha)^2\Exp\|Y-\Exp[Y]\|^2 +\ \alpha \underbrace{\Exp[\langle X - \Exp[X], Y - \Exp[Y] \rangle]}_{=0}\\
&= \alpha^2\Variance[X] + (1-\alpha)^2\Variance[Y],
\end{align*}
which establishes the first equality. The second equality follows from decompsing each of $X = \nabbarone F(\btheta)$ and $Y = \nabbarone F(\btheta)$ as the empirical mean of $N$ i.i.d random variables. 
\end{proof}
The following lemma justifies using $\left( \alpha^2\sighatsq_1 + (1-\alpha)^2 \sighatsq_2\right)$ as a proxy for the variance: 
\begin{lemma}[Empirical variance]\label{lem:empirical_variance}  For $k = 0,1$, we have
\begin{align*}
\frac{1}{N}\Exp[\sighatsq_k] = \Variance[\nabbar^{[k]}].
\end{align*}
Thus,
\begin{align*}
\Exp[\left( \alpha^2\sighatsq_1 + (1-\alpha)^2 \sighatsq_2\right)] = N \cdot \Variance[\nabbaralpha].
\end{align*}
\end{lemma}
\begin{proof} The first part of the lemma follows from a standard unbiasedness computation for a sample variance (see, e.g. \citet[Theorem 3.17]{wasserman2004all} for the scalar case). The second part of the lemma follows from \Cref{lem:interpolated_variance}. 
\end{proof}

\subsection{Closed-form for interpolation}\label{app:closed_form}
Recall \Cref{lem:closed_form}:
    With $\gamma = \infty$, the optimal $\alpha$ is $\alpha_{\infty} := \frac{\sigma_0^2}{\sigma_1^2 + \sigma_0^2}$. For finite $\gamma \ge \epsilon$, \Cref{eq:tractableform} is 
    \begin{equation}
        \alpha_{\gamma} := \begin{cases}
            \alpha_{\infty} & \text{ if } \quad \alpha_{\infty} B \leq \gamma - \varepsilon \\ 
            \frac{\gamma-\varepsilon}{B} & \text{ otherwise }. 
        \end{cases}
    \end{equation}
    \begin{proof}
        Intuitively, the objective is convex with a linear constraint, so meets its optimality either at the unconstrained minimum or at the constraint surface. This is implied by complementary slackness of the KKT conditions, since an optimal $\alpha^*$ satisfies:
        \begin{align*}
            2\alpha^* \hat{\sigma}^2_1 + 2(1-\alpha^*) \hat{\sigma}^2_0 + \lambda B = 0 \\
            \lambda (\epsilon - \gamma + \alpha^* B)=0,
        \end{align*}
        where the first line is stationarity of the Lagrangian and the second line is complementary slackness. Clearly, either $\lambda=0$ and the minimum is met at the inverse-weighted solution of the variances, or the constraint is zero and we have $\alpha^*=(\gamma - \epsilon)/B$.
    \end{proof}

\subsection{Proof of \Cref{lem:robustness}}\label{app:triangle}
We give a more detailed proof of \Cref{lem:robustness} here. 
\begin{align*}
&\|\nabbaralpha F(\btheta) - \nabla F(\btheta) \| \\
& = \|\alpha \nabbarone F(\btheta) + (1-\alpha) \nabbarzero F(\btheta) - \nabla F(\btheta) \| \\
& = \|\alpha \nabbarone F(\btheta) + (1-\alpha) \nabbarzero F(\btheta) - \alpha \nabla F(\btheta) - (1-\alpha) \nabla F(\btheta) \| \\
&\le (1-\alpha)\| \nabbarzero F(\btheta)- \nabla F(\btheta)\| + \alpha \| \nabbarone F(\btheta)- \nabla F(\btheta)\| \\
&\le (1-\alpha)\| \nabbarzero F(\btheta)- \nabla F(\btheta)\| + \alpha \left(\|\nabbarone F(\btheta)- \nabbarzero F(\btheta)\| + \|\nabbarzero F(\btheta) - \nabla F(\btheta)\|\right) \\
&\le \| \nabbarzero F(\btheta)- \nabla F(\btheta)\| + \alpha \| \nabbarone F(\btheta)- \nabbarzero F(\btheta)\|\\
&\le \epsilon + \alpha \| \nabbarone F(\btheta)- \nabbarzero F(\btheta)\|\\
& \le \gamma.
\end{align*}

\subsection{Empirical Bernstein confidence}\label{app:bernstein}
Here describe our confidence estimate based on the \zobg{}. Recall that that \zobg{} is
\begin{align*}
\nabbarzero F(\btheta) = \frac{1}{N} \sum_{i=1}^N \nabhatzero F_i(\btheta), \quad \text{where} \quad \nabhatzero F_i(\btheta) \sum_{i=1}^N V_1(\bx_1^i,\bw_{1:H}^i,\btheta)\cdot \left[\sum_{i=1}^H \rmD_{\btheta} \pi(\bx_h,\btheta)^\top \bw_h \right].
\end{align*}
Our estimate is based on the matrix Bernstein inequality due (see, e.g. \cite{matrixconcentration}) specified  below. 
\begin{lemma}[Matrix Bernstein inequality]\label{lem:Bernstein} Let $X_1,\dots,X_N$ be $N$ i.i.d random $d$-dimensional random \emph{vectors} with $\|X_1 - \Exp[X_1]\| \le R$ almost surely, and $\Exp[\|X_1\|^2] \le \sigma^2$. Then,
\begin{align*}
\Pr\left[ \left\|\frac{1}{N}\sum_{i=1}^N X_i - \Exp[X]\right\| \ge t \right] \le (d+1)\exp\left(\frac{-Nt^2/2}{\sigma^2 + Rt/3}\right)
\end{align*}
Hence, with probability, for any $\delta > 0$,
\begin{align*}
\Pr\left[\left\|\frac{1}{N}\sum_{i=1}^N X_i - \Exp[X]\right\| \ge  \sqrt{\frac{2\sigma^2\log\frac{d+1}{\delta}}{N}} + \frac{2R}{3N}\log\frac{d+1}{\delta}\right] \le 1- \delta. 
\end{align*}
\end{lemma}
As stated, \Cref{lem:Bernstein} does not apply to our setting because (a) the variance of each $X_i := \nabhatzero F_i(\btheta)$ is unknown, and (b) $X_i$ are not uniformly bounded (due to the Gaussian noise $\bw_h^i$ being unbounded.) We address point (a) by replacing $\Variance[X_i]$ with the following empirical upper bound
\begin{align*}
\bar{\sigma}_{0}^2 := \sum_{i}\|\nabhatzero F_i(\btheta)\|^2 \ge \sighatsq_0.
\end{align*}
To address point (b), we take $R$ to be some educated guess on the problem using the gradient samples from the system (e.g. $R=\max_i \| \nabhatzero F_i(\btheta)-\nabbarzero F(\btheta)\|$). In practice, since the confidence bound $\epsilon$ directly scales with $R$, and the user needs to set some threshold term $\gamma$ on $\epsilon + \alpha B$, a guess on the scale of $R$ is already decided by the user threshold $\gamma$. Thus, rather than viewing $R$ as a rigorous absolute bound on the max deviation that we have to compute, we interpret it as a hyperparameter balancing how much we should be cautious against an extreme deviation outside the events covered by the variance term. We find that this approach, while not entirely rigorous, performs well in simulation. The following remark sketches how a rigorous confidence interval could be derived.  
\begin{remark} For a statistically rigorous confidence interval, one would have to (a) control the error introduced by using an empirical estimate of the variance, and (b) control the non-boundedness of the $X_i$ vectors. The first point could be addressed by generalizing the empirical Bernstein inequality \cite{maurer2009empirical} (which slightly inflates the confidence intervals to accomodate fluctuations in empirical variance) to vector-valued random variables. Point (b) can be handled by a truncation argument, leveraging the light-tails of Gaussian vectors. Nevertheless, we find that our naive approach which substitutes in the empirical variance for the true variance and our choice of $R$ has good performance in simulation, so we do not pursue more complicated machinery. In fact, we conjecture that a more rigorous concentration bound may be overly conservative and worse in experiments. 
\end{remark}

%% file: icml_main.bbl
\begin{thebibliography}{55}
\providecommand{\natexlab}[1]{#1}
\providecommand{\url}[1]{\texttt{#1}}
\expandafter\ifx\csname urlstyle\endcsname\relax
  \providecommand{\doi}[1]{doi: #1}\else
  \providecommand{\doi}{doi: \begingroup \urlstyle{rm}\Url}\fi

\bibitem[Agarwal et~al.(2020)Agarwal, Kakade, Lee, and Mahajan]{agarwal}
Agarwal, A., Kakade, S.~M., Lee, J.~D., and Mahajan, G.
\newblock On the theory of policy gradient methods: Optimality, approximation,
  and distribution shift, 2020.

\bibitem[Bangaru et~al.(2021)Bangaru, Michel, Mu, Bernstein, Li, and
  Ragan-Kelley]{teg}
Bangaru, S.~P., Michel, J., Mu, K., Bernstein, G., Li, T.-M., and Ragan-Kelley,
  J.
\newblock Systematically differentiating parametric discontinuities.
\newblock \emph{ACM Trans. Graph.}, 40\penalty0 (4), July 2021.
\newblock ISSN 0730-0301.
\newblock \doi{10.1145/3450626.3459775}.

\bibitem[Berahas et~al.(2019)Berahas, Cao, Choromanski, and
  Scheinberg]{gradientapproximation}
Berahas, A.~S., Cao, L., Choromanski, K., and Scheinberg, K.
\newblock A theoretical and empirical comparison of gradient approximations in
  derivative-free optimization.
\newblock \emph{arXiv: Optimization and Control}, 2019.

\bibitem[Bhandari \& Russo(2020)Bhandari and Russo]{russo}
Bhandari, J. and Russo, D.
\newblock Global optimality guarantees for policy gradient methods, 2020.

\bibitem[Boyd \& Vandenberghe(2004)Boyd and Vandenberghe]{boyd}
Boyd, S. and Vandenberghe, L.
\newblock \emph{Convex optimization}.
\newblock Cambridge university press, 2004.

\bibitem[Carpentier et~al.(2019)Carpentier, Saurel, Buondonno, Mirabel,
  Lamiraux, Stasse, and Mansard]{pinocchio}
Carpentier, J., Saurel, G., Buondonno, G., Mirabel, J., Lamiraux, F., Stasse,
  O., and Mansard, N.
\newblock The pinocchio c++ library : A fast and flexible implementation of
  rigid body dynamics algorithms and their analytical derivatives.
\newblock In \emph{2019 IEEE/SICE International Symposium on System Integration
  (SII)}, pp.\  614--619, 2019.
\newblock \doi{10.1109/SII.2019.8700380}.

\bibitem[Castro et~al.(2020)Castro, Qu, Kuppuswamy, Alspach, and
  Sherman]{tamsi}
Castro, A.~M., Qu, A., Kuppuswamy, N., Alspach, A., and Sherman, M.
\newblock A transition-aware method for the simulation of compliant contact
  with regularized friction.
\newblock \emph{IEEE Robotics and Automation Letters}, 5\penalty0 (2):\penalty0
  1859–1866, Apr 2020.
\newblock ISSN 2377-3774.
\newblock \doi{10.1109/lra.2020.2969933}.
\newblock URL \url{http://dx.doi.org/10.1109/LRA.2020.2969933}.

\bibitem[{\c C}inlar(2011)]{cinlar2011probability}
{\c C}inlar, E.
\newblock \emph{Probability and stochastics}, volume 261.
\newblock Springer, 2011.

\bibitem[Coumans \& Bai(2016--2021)Coumans and Bai]{bullet}
Coumans, E. and Bai, Y.
\newblock Pybullet, a python module for physics simulation for games, robotics
  and machine learning.
\newblock \url{http://pybullet.org}, 2016--2021.

\bibitem[de~Avila Belbute-Peres et~al.(2018)de~Avila Belbute-Peres, Smith,
  Allen, Tenenbaum, and Kolter]{filipe}
de~Avila Belbute-Peres, F., Smith, K., Allen, K., Tenenbaum, J., and Kolter,
  J.~Z.
\newblock End-to-end differentiable physics for learning and control.
\newblock In Bengio, S., Wallach, H., Larochelle, H., Grauman, K.,
  Cesa-Bianchi, N., and Garnett, R. (eds.), \emph{Advances in Neural
  Information Processing Systems}, volume~31. Curran Associates, Inc., 2018.
\newblock URL
  \url{https://proceedings.neurips.cc/paper/2018/file/842424a1d0595b76ec4fa03c46e8d755-Paper.pdf}.

\bibitem[Du et~al.(2020)Du, Li, Xu, Spielberg, Wu, Rus, and Matusik]{d3pg}
Du, T., Li, Y., Xu, J., Spielberg, A., Wu, K., Rus, D., and Matusik, W.
\newblock D3{\{}pg{\}}: Deep differentiable deterministic policy gradients,
  2020.
\newblock URL \url{https://openreview.net/forum?id=rkxZCJrtwS}.

\bibitem[Duchi et~al.(2011)Duchi, Bartlett, and Wainwright]{duchi2}
Duchi, J., Bartlett, P., and Wainwright, M.
\newblock Randomized smoothing for stochastic optimization.
\newblock \emph{SIAM Journal on Optimization}, 22, 03 2011.
\newblock \doi{10.1137/110831659}.

\bibitem[Duchi et~al.(2015)Duchi, Jordan, Wainwright, and Wibisono]{duchi1}
Duchi, J., Jordan, M., Wainwright, M., and Wibisono, A.
\newblock Optimal rates for zero-order convex optimization: The power of two
  function evaluations.
\newblock \emph{IEEE Transactions on Information Theory}, 61, 12 2015.
\newblock \doi{10.1109/TIT.2015.2409256}.

\bibitem[Elandt et~al.(2019)Elandt, Drumwright, Sherman, and
  Ruina]{Elandt2019APF}
Elandt, R., Drumwright, E., Sherman, M., and Ruina, A.
\newblock A pressure field model for fast, robust approximation of net contact
  force and moment between nominally rigid objects.
\newblock \emph{IROS}, pp.\  8238--8245, 2019.

\bibitem[Ern \& Guermond(2013)Ern and Guermond]{ern2013theory}
Ern, A. and Guermond, J.-L.
\newblock \emph{Theory and practice of finite elements}, volume 159.
\newblock Springer Science \& Business Media, 2013.

\bibitem[Fazel et~al.(2019)Fazel, Ge, Kakade, and Mesbahi]{fazel}
Fazel, M., Ge, R., Kakade, S.~M., and Mesbahi, M.
\newblock Global convergence of policy gradient methods for the linear
  quadratic regulator, 2019.

\bibitem[Freeman et~al.(2021)Freeman, Frey, Raichuk, Girgin, Mordatch, and
  Bachem]{brax}
Freeman, C.~D., Frey, E., Raichuk, A., Girgin, S., Mordatch, I., and Bachem, O.
\newblock Brax - a differentiable physics engine for large scale rigid body
  simulation.
\newblock In \emph{Thirty-fifth Conference on Neural Information Processing
  Systems Datasets and Benchmarks Track (Round 1)}, 2021.
\newblock URL \url{https://openreview.net/forum?id=VdvDlnnjzIN}.

\bibitem[Geilinger et~al.(2020)Geilinger, Hahn, Zehnder, Bächer, Thomaszewski,
  and Coros]{add}
Geilinger, M., Hahn, D., Zehnder, J., Bächer, M., Thomaszewski, B., and Coros,
  S.
\newblock Add: Analytically differentiable dynamics for multi-body systems with
  frictional contact, 2020.

\bibitem[Ghadimi \& Lan(2013)Ghadimi and Lan]{ghadimilan}
Ghadimi, S. and Lan, G.
\newblock Stochastic first- and zeroth-order methods for nonconvex stochastic
  programming.
\newblock \emph{SIAM Journal on Optimization}, 23\penalty0 (4):\penalty0
  2341--2368, 2013.
\newblock \doi{10.1137/120880811}.
\newblock URL \url{https://doi.org/10.1137/120880811}.

\bibitem[Gradu et~al.(2021)Gradu, Hallman, Suo, Yu, Agarwal, Ghai, Singh,
  Zhang, Majumdar, and Hazan]{deluca}
Gradu, P., Hallman, J., Suo, D., Yu, A., Agarwal, N., Ghai, U., Singh, K.,
  Zhang, C., Majumdar, A., and Hazan, E.
\newblock Deluca -- a differentiable control library: Environments, methods,
  and benchmarking, 2021.

\bibitem[Howell et~al.(2022)Howell, Cleac'h, Kolter, Schwager, and
  Manchester]{dojo}
Howell, T.~A., Cleac'h, S.~L., Kolter, J.~Z., Schwager, M., and Manchester, Z.
\newblock Dojo: A differentiable simulator for robotics, 2022.
\newblock URL \url{https://arxiv.org/abs/2203.00806}.

\bibitem[Hu et~al.(2020)Hu, Anderson, Li, Sun, Carr, Ragan-Kelley, and
  Durand]{difftaichi}
Hu, Y., Anderson, L., Li, T.-M., Sun, Q., Carr, N., Ragan-Kelley, J., and
  Durand, F.
\newblock Difftaichi: Differentiable programming for physical simulation.
\newblock \emph{ICLR}, 2020.

\bibitem[Huang et~al.(2021)Huang, Hu, Du, Zhou, Su, Tenenbaum, and
  Gan]{plasticinelab}
Huang, Z., Hu, Y., Du, T., Zhou, S., Su, H., Tenenbaum, J.~B., and Gan, C.
\newblock Plasticinelab: A soft-body manipulation benchmark with differentiable
  physics.
\newblock In \emph{International Conference on Learning Representations}, 2021.
\newblock URL \url{https://openreview.net/forum?id=xCcdBRQEDW}.

\bibitem[Hunt \& Crossley(1975)Hunt and Crossley]{huntcrossley}
Hunt, K.~H. and Crossley, F. R.~E.
\newblock {Coefficient of Restitution Interpreted as Damping in Vibroimpact}.
\newblock \emph{Journal of Applied Mechanics}, 42\penalty0 (2):\penalty0
  440--445, 06 1975.
\newblock ISSN 0021-8936.
\newblock \doi{10.1115/1.3423596}.
\newblock URL \url{https://doi.org/10.1115/1.3423596}.

\bibitem[Kakade(2003)]{kakade2003sample}
Kakade, S.~M.
\newblock \emph{On the sample complexity of reinforcement learning}.
\newblock University of London, University College London (United Kingdom),
  2003.

\bibitem[Kingma et~al.(2015)Kingma, Salimans, and Welling]{kingma}
Kingma, D.~P., Salimans, T., and Welling, M.
\newblock Variational dropout and the local reparameterization trick.
\newblock In Cortes, C., Lawrence, N., Lee, D., Sugiyama, M., and Garnett, R.
  (eds.), \emph{Advances in Neural Information Processing Systems}, volume~28.
  Curran Associates, Inc., 2015.

\bibitem[Lasota \& Mackey(1996)Lasota and Mackey]{chaos}
Lasota, A. and Mackey, M.~C.
\newblock \emph{Chaos, Fractals, and Noise: Stochastic Aspects of Dynamics}.
\newblock Cambridge university press, 1996.

\bibitem[Le~Lidec et~al.(2021)Le~Lidec, Montaut, Schmid, Laptev, and
  Carpentier]{randomizedsmoothing}
Le~Lidec, Q., Montaut, L., Schmid, C., Laptev, I., and Carpentier, J.
\newblock {Leveraging Randomized Smoothing for Optimal Control of Nonsmooth
  Dynamical Systems}.
\newblock working paper or preprint, December 2021.
\newblock URL \url{https://hal.archives-ouvertes.fr/hal-03480419}.

\bibitem[Macklin et~al.(2014)Macklin, M\"{u}ller, Chentanez, and Kim]{flex}
Macklin, M., M\"{u}ller, M., Chentanez, N., and Kim, T.-Y.
\newblock Unified particle physics for real-time applications.
\newblock \emph{ACM Trans. Graph.}, 33\penalty0 (4), jul 2014.
\newblock ISSN 0730-0301.
\newblock \doi{10.1145/2601097.2601152}.
\newblock URL \url{https://doi-org.libproxy.mit.edu/10.1145/2601097.2601152}.

\bibitem[Mahamed et~al.(2020)Mahamed, Rosca, Figurnov, and Mnih]{montecarlo}
Mahamed, S., Rosca, M., Figurnov, M., and Mnih, A.
\newblock Monte carlo gradient estimation in machine learning.
\newblock In Dy, J. and Krause, A. (eds.), \emph{Journal of Machine Learning
  Research}, volume~21, pp.\  1--63, 4 2020.

\bibitem[Mason(2001)]{robotmanipulation}
Mason, M.~T.
\newblock \emph{{Mechanics of Robotic Manipulation}}.
\newblock The MIT Press, 06 2001.
\newblock ISBN 9780262256629.
\newblock \doi{10.7551/mitpress/4527.001.0001}.
\newblock URL \url{https://doi.org/10.7551/mitpress/4527.001.0001}.

\bibitem[Maurer \& Pontil(2009)Maurer and Pontil]{maurer2009empirical}
Maurer, A. and Pontil, M.
\newblock Empirical bernstein bounds and sample variance penalization.
\newblock \emph{arXiv preprint arXiv:0907.3740}, 2009.

\bibitem[Metz et~al.(2019)Metz, Maheswaranathan, Nixon, Freeman, and
  Sohl-Dickstein]{pathologies}
Metz, L., Maheswaranathan, N., Nixon, J., Freeman, D., and Sohl-Dickstein, J.
\newblock Understanding and correcting pathologies in the training of learned
  optimizers.
\newblock In Chaudhuri, K. and Salakhutdinov, R. (eds.), \emph{Proceedings of
  the 36th International Conference on Machine Learning}, volume~97 of
  \emph{Proceedings of Machine Learning Research}, pp.\  4556--4565. PMLR,
  09--15 Jun 2019.

\bibitem[Metz et~al.(2021)Metz, Freeman, Schoenholz, and
  Kachman]{metz2021gradients}
Metz, L., Freeman, C.~D., Schoenholz, S.~S., and Kachman, T.
\newblock Gradients are not all you need, 2021.

\bibitem[Mirtich(1996)]{brianmirtich}
Mirtich, B.~V.
\newblock \emph{Impulse-Based Dynamic Simulation of Rigid Body Systems}.
\newblock PhD thesis, 1996.
\newblock AAI9723116.

\bibitem[Mora et~al.(2021)Mora, Peychev, Ha, Vechev, and Coros]{pods}
Mora, M. A.~Z., Peychev, M., Ha, S., Vechev, M., and Coros, S.
\newblock Pods: Policy optimization via differentiable simulation.
\newblock In Meila, M. and Zhang, T. (eds.), \emph{Proceedings of the 38th
  International Conference on Machine Learning}, volume 139 of
  \emph{Proceedings of Machine Learning Research}, pp.\  7805--7817. PMLR,
  18--24 Jul 2021.
\newblock URL \url{https://proceedings.mlr.press/v139/mora21a.html}.

\bibitem[Pang(2021)]{pangsim}
Pang, T.
\newblock A convex quasistatic time-stepping scheme for rigid multibody systems
  with contact and friction.
\newblock \emph{2021 IEEE International Conference on Robotics and Automation
  (ICRA)}, pp.\  6614--6620, 2021.

\bibitem[Parmas et~al.(2018)Parmas, Rasmussen, Peters, and Doya]{parmas}
Parmas, P., Rasmussen, C.~E., Peters, J., and Doya, K.
\newblock {PIPPS}: Flexible model-based policy search robust to the curse of
  chaos.
\newblock In Dy, J. and Krause, A. (eds.), \emph{Proceedings of the 35th
  International Conference on Machine Learning}, volume~80 of \emph{Proceedings
  of Machine Learning Research}, pp.\  4065--4074. PMLR, 10--15 Jul 2018.

\bibitem[Paszke et~al.(2019)Paszke, Gross, Massa, Lerer, Bradbury, Chanan,
  Killeen, Lin, Gimelshein, Antiga, Desmaison, Kopf, Yang, DeVito, Raison,
  Tejani, Chilamkurthy, Steiner, Fang, Bai, and Chintala]{pytorch}
Paszke, A., Gross, S., Massa, F., Lerer, A., Bradbury, J., Chanan, G., Killeen,
  T., Lin, Z., Gimelshein, N., Antiga, L., Desmaison, A., Kopf, A., Yang, E.,
  DeVito, Z., Raison, M., Tejani, A., Chilamkurthy, S., Steiner, B., Fang, L.,
  Bai, J., and Chintala, S.
\newblock Pytorch: An imperative style, high-performance deep learning library.
\newblock In Wallach, H., Larochelle, H., Beygelzimer, A., d'Alch\'{e} Buc, F.,
  Fox, E., and Garnett, R. (eds.), \emph{Advances in Neural Information
  Processing Systems 32}, pp.\  8024--8035. Curran Associates, Inc., 2019.

\bibitem[Rudin et~al.(1964)]{rudin1964principles}
Rudin, W. et~al.
\newblock \emph{Principles of mathematical analysis}, volume~3.
\newblock McGraw-hill New York, 1964.

\bibitem[Schulman et~al.(2015)Schulman, Heess, Weber, and
  Abbeel]{stochasticcompuationgraph}
Schulman, J., Heess, N., Weber, T., and Abbeel, P.
\newblock Gradient estimation using stochastic computation graphs.
\newblock In Cortes, C., Lawrence, N., Lee, D., Sugiyama, M., and Garnett, R.
  (eds.), \emph{Advances in Neural Information Processing Systems}, volume~28.
  Curran Associates, Inc., 2015.

\bibitem[Schulman et~al.(2017)Schulman, Wolski, Dhariwal, Radford, and
  Klimov]{ppo}
Schulman, J., Wolski, F., Dhariwal, P., Radford, A., and Klimov, O.
\newblock Proximal policy optimization algorithms, 2017.

\bibitem[Stein \& Shakarchi(2009)Stein and Shakarchi]{stein2009real}
Stein, E.~M. and Shakarchi, R.
\newblock \emph{Real analysis}.
\newblock Princeton University Press, 2009.

\bibitem[Stewart \& Trinkle(2000)Stewart and Trinkle]{stewarttrinkle}
Stewart, D. and Trinkle, J.~J.
\newblock An implicit time-stepping scheme for rigid body dynamics with coulomb
  friction.
\newblock volume~1, pp.\  162--169, 01 2000.
\newblock \doi{10.1109/ROBOT.2000.844054}.

\bibitem[Stribeck(1903)]{stribeck}
Stribeck, R.
\newblock \emph{Die wesentlichen Eigenschaften der Gleit- und Rollenlager}.
\newblock Mitteilungen {\"u}ber Forschungsarbeiten auf dem Gebiete des
  Ingenieurwesens, insbesondere aus den Laboratorien der technischen
  Hochschulen. Julius Springer, 1903.

\bibitem[Suh et~al.(2021)Suh, Pang, and Tedrake]{suh2021bundled}
Suh, H. J.~T., Pang, T., and Tedrake, R.
\newblock Bundled gradients through contact via randomized smoothing.
\newblock \emph{arXiv pre-print}, 2021.

\bibitem[Sutton et~al.(2000)Sutton, Mcallester, Singh, and
  Mansour]{policygradienttheorem}
Sutton, R., Mcallester, D., Singh, S., and Mansour, Y.
\newblock Policy gradient methods for reinforcement learning with function
  approximation.
\newblock \emph{Adv. Neural Inf. Process. Syst}, 12, 02 2000.

\bibitem[Tedrake(2022)]{drake}
Tedrake, R.
\newblock Drake: A planning, control, and analysis toolbox for nonlinear
  dynamical systems, 2022.
\newblock URL \url{http://drake.mit.edu}.

\bibitem[Todorov et~al.(2012)Todorov, Erez, and Tassa]{mujoco}
Todorov, E., Erez, T., and Tassa, Y.
\newblock Mujoco: A physics engine for model-based control.
\newblock In \emph{IROS}, pp.\  5026--5033, 2012.
\newblock \doi{10.1109/IROS.2012.6386109}.

\bibitem[Tropp(2015)]{matrixconcentration}
Tropp, J.~A.
\newblock An introduction to matrix concentration inequalities.
\newblock \emph{Foundations and Trends® in Machine Learning}, 8\penalty0
  (1-2):\penalty0 1--230, 2015.
\newblock ISSN 1935-8237.
\newblock \doi{10.1561/2200000048}.
\newblock URL \url{http://dx.doi.org/10.1561/2200000048}.

\bibitem[van~der Schaft \& Schumacher(2000)van~der Schaft and
  Schumacher]{hybridsystems}
van~der Schaft, A. and Schumacher, H.
\newblock \emph{An Introduction to Hybrid Dynamical Systems}.
\newblock Springer Publishing Company, Incorporated, 1st edition, 2000.
\newblock ISBN 978-1-4471-3916-4.

\bibitem[Wasserman(2004)]{wasserman2004all}
Wasserman, L.
\newblock \emph{All of statistics: a concise course in statistical inference},
  volume~26.
\newblock Springer, 2004.

\bibitem[Werling et~al.(2021)Werling, Omens, Lee, Exarchos, and Liu]{werling}
Werling, K., Omens, D., Lee, J., Exarchos, I., and Liu, C.~K.
\newblock Fast and feature-complete differentiable physics for articulated
  rigid bodies with contact, 2021.

\bibitem[Williams(1992)]{reinforce}
Williams, R.~J.
\newblock Simple statistical gradient-following algorithms for connectionist
  reinforcement learning.
\newblock \emph{Machine Learning}, 3, 05 1992.

\bibitem[Zhang et~al.(2020)Zhang, Koppel, Zhu, and Başar]{kaiqing}
Zhang, K., Koppel, A., Zhu, H., and Başar, T.
\newblock Global convergence of policy gradient methods to (almost) locally
  optimal policies, 2020.

\end{thebibliography}
